\documentclass[11pt]{article}

\usepackage{graphicx}
\usepackage{subfigure}
\usepackage{amsmath, amsthm, amsfonts}
\usepackage{appendix}
\usepackage{cite}
\usepackage{fullpage}
\usepackage{hyperref}

\newcommand{\domain}{\Theta}
\renewcommand{\th}{h'}
\newcommand{\hk}{h^k}
\newcommand{\yk}{y^k}
\newcommand{\pk}{p^k}
\newcommand{\tw}{w_{\mathrm{trunc}}}
\newcommand{\rw}{w_{\mathrm{recurr}}}

\newcommand{\todo}[1]{}
\renewcommand{\todo}[1]{{\color{red} TODO: {#1}}}

\newcommand{\reals}{{\mbox{\bf R}}}

\newcommand{\eps}{\varepsilon}

\newcommand{\pder}[2]{\frac{\partial#1}{\partial#2}} 
\newcommand{\der}[2]{\frac{d}{d#2}} 
\newcommand{\grad}{\nabla}

\newcommand{\diag}{\mathop{\bf diag}}

\newcommand{\1}[1]{\ensuremath{\mathbb{I}} \left[#1\right]} 
\newcommand{\simiid}{\overset{\textrm{i.i.d.}}{\sim}}
\newcommand{\normal}[2]{\mathcal N \left(#1, #2\right)}

\def\set#1{\left\{ #1 \right\}}
\def\abs#1{\left| #1 \right|}
\def\1#1{1\left[ #1 \right]} 

\def\norm#1{\left\| #1 \right\|}
\def\paren#1{\left( #1 \right)}     
\def\brack#1{\left[ #1 \right]}     

\newtheorem{thm}{Theorem}
\newtheorem*{thm*}{Theorem}
\newtheorem{prop}{Proposition}
\newtheorem*{prop*}{Proposition}
\newtheorem{lemma}{Lemma}

\newtheorem{defn}{Definition}

\title{Stable Recurrent Models}

\author{
  John Miller\thanks{Email: miller\_john@berkeley.edu} \and
  Moritz Hardt\thanks{Email: hardt@berkeley.edu}}

\begin{document}
\maketitle

\begin{abstract}

Stability is a fundamental property of dynamical systems, yet to this date it
has had little bearing on the practice of recurrent neural networks. In this
work, we conduct a thorough investigation of stable recurrent models.
Theoretically, we prove stable recurrent neural networks are well approximated
by feed-forward networks for the purpose of both inference and training by
gradient descent. Empirically, we demonstrate stable recurrent models often
perform as well as their unstable counterparts on benchmark sequence tasks.
Taken together, these findings shed light on the effective power of recurrent
networks and suggest much of sequence learning happens, or can be made to
happen, in the stable regime. Moreover, our results help to explain why in many
cases practitioners succeed in replacing recurrent models by feed-forward
models.

\end{abstract}

\newcommand{\remove}[1]{}

\section{Introduction}
Recurrent neural networks are a popular modeling choice for solving sequence
learning problems arising in domains such as speech recognition and natural
language processing. At the outset, recurrent neural networks are non-linear
dynamical systems commonly trained to fit sequence data via some variant of
gradient descent. 

Stability is of fundamental importance in the study of dynamical system.
Surprisingly, however, stability has had little impact on the practice of
recurrent neural networks.  Recurrent models trained in practice do not satisfy
stability in an obvious manner, suggesting that perhaps training happens in a
chaotic regime. The difficulty of training recurrent models has compelled
practitioners to successfully replace recurrent models with non-recurrent,
feed-forward architectures.

This state of affairs raises important unresolved questions. 
\begin{center} 
\emph{Is sequence
modeling in practice inherently unstable? When and why are recurrent models
really needed?}
\end{center}

In this work, we shed light on both of these questions through a theoretical and
empirical investigation of stability in recurrent models.

We first prove stable recurrent models can be approximated by
feed-forward networks. In particular, not only are the models equivalent for
\emph{inference}, they are also equivalent for \emph{training} via gradient
descent.  While it is easy to contrive non-linear recurrent models that on some
input sequence cannot be approximated by feed-forward models, our result implies
such models are inevitably unstable. This means in particular they must have
exploding gradients, which is in general an impediment to learnibility via
gradient descent.

Second, across a variety of different sequence tasks, we show how
\emph{recurrent models can often be made stable without loss in performance}.
We also show models that are nominally unstable often operate in the stable
regime on the data distribution. Combined with our first result, these
observation helps to explain why an increasingly large body of empirical
research succeeds in replacing recurrent models with feed-forward models in
important applications, including translation~\cite{vaswani2017attention,
gehring2017convolutional}, speech synthesis~\cite{van2016wavenet}, and language
modeling~\cite{dauphin2017language}.  While stability does not always hold in
practice to begin with, it is often possible to generate a high-performing
stable model by \emph{imposing stability} during training. 

Our results also shed light on the effective representational properties of
recurrent networks trained in practice. In particular, stable models cannot have
long-term memory. Therefore, when stable and unstable models achieve similar
results, either the task does not require long-term memory, or the unstable
model does not have it.

\remove{At first glance, this approximation result should give one pause about
stable models. It easy to contrive non-linear recurrent models that on some
input sequences cannot be approximated by feed-forward models. But would such
recurrent models be trainable by gradient descent? After all, stable recurrent
models are precisely those models where gradient descent is expected to work.}

\remove{
A growing line of theoretical research has sought to understand when and why
recurrent models are trainable by gradient descent \cite{hardt2016gradient,
oymak2018stochastic}. Characterizing exactly which recurrent models are
learnable by gradient descent is a delicate task beyond the reach of current
theory. Consequently, many of these works identify \emph{stability} as a
natural criterion for training recurrent models. Roughly speaking, stability
is the requirement that the gradients of the training objective do not
\emph{explode} over time.

While stability offers convenient theoretical properties, recurrent models
trained in practice typically do not satisfy stability. This tension between
theory and practice raises an intriguing question for investigation:
\begin{center}
    \emph{Are stable recurrent models merely a theoretically convenient
        construct, or do they capture some aspect of practice?}
\end{center}
Answering this question requires fully understanding both the power and
limitations of stable recurrent models. Our work investigates both the
theoretical consequences of stability and the empirical performance of stable
recurrent models on variety of benchmark sequence tasks.
}

\subsection{Contributions} 
In this work, we make the following contributions.
\begin{enumerate}
\item 
We present a generic definition of stable recurrent models in terms of
non-linear dynamical systems and show how to ensure stability of several
commonly used models. Previous work establishes stability for vanilla
recurrent neural networks. We give new sufficient conditions for stability of
long short-term memory (LSTM) networks. These sufficient conditions come with
an efficient projection operator that can be used at training time to enforce
stability.
\item We prove, under the stability assumption, feed-forward networks can
approximate recurrent networks for purposes of both inference and training by
gradient descent. While simple in the case of inference, the training result
relies on non-trivial stability properties of gradient descent.
\item We conduct extensive experimentation on a variety of sequence benchmarks, 
show stable models often have comparable performance with their unstable
counterparts, and discuss when, if ever, there is an intrinsic performance price
to using stable models.
\end{enumerate}

\section{Stable Recurrent Models}
\label{sec:stable_rnns}
In this section, we define \emph{stable recurrent models} and illustrate the
concept for various popular model classes.  From a pragmatic perspective,
stability roughly corresponds to the criterion that the gradients of the
training objective do not \emph{explode} over time. Common recurrent models can
operate in both the stable and unstable regimes, depending on their parameters.
To study stable variants of common architectures, we give sufficient conditions
to ensure stability and describe how to efficiently enforce these conditions during
training.

\subsection{Defining Stable Recurrent Models}
A \emph{recurrent model} is a non-linear dynamical system given by a differentiable
\emph{state-transition map}
$\phi_w \colon \reals^n \times \reals^d \to \reals^n$,
parameterized by $w \in \reals^m.$ The hidden state $h_t\in\reals^n$ evolves in
discrete time steps according to the update rule
\begin{align}
    h_t = \phi_w(h_{t-1}, x_t)\,, \label{eq:full-system}
\end{align}
where the vector~$x_t \in \reals^d$ is an arbitrary input provided to the system
at time~$t$. This general formulation allows us to unify many examples of
interest. For instance, for a recurrent neural network, given
weight matrices $W$ and $U$, the state evolves according to
\begin{align*}
    h_t = \phi_{W, U}(h_{t-1}, x_t) = \tanh\paren{W h_{t-1} + U x_t}.
\end{align*}

Recurrent models are typically trained using some variant of gradient descent.
One natural---even if not strictly necessary---requirement for gradient descent
to work is that the gradients of the training objective do not explode over time.  
\emph{Stable recurrent models} are precisely
the class of models where the gradients cannot explode. They thus constitute a
natural class of models where gradient descent can be expected to work. 
In general, we define a stable recurrent model as follows.
\begin{defn}
    A recurrent model $\phi_w$ is \emph{stable} if there exists some $\lambda < 1$
    such that, for any weights $w \in \reals^m$, states $h, h' \in \reals^n$,
    and input $x \in \reals^d$,
    \begin{align} \label{def:stability}
        \norm{\phi_w(h, x) - \phi_w(h', x)} \leq \lambda \norm{h - h'}.
    \end{align}
\end{defn}
Equivalently, a recurrent model is stable if the map $\phi_w$ is
$\lambda$-contractive in $h$. If $\phi_w$ is $\lambda$-stable, then
$\norm{\nabla_h \phi_w(h, x)} < \lambda$, and for Lipschitz loss $p$,
$\norm{\grad_w p}$ is always bounded \cite{pascanu2013difficulty}.

Stable models are particularly well-behaved and well-justified from a theoretical
perspective. For instance, at present, only \emph{stable} linear dynamical
systems are known to be learnable via gradient descent
\cite{hardt2018gradient}. In unstable models, the gradients of the objective can
explode, and it is a delicate matter to even show that
gradient descent converges to a stationary point. The following proposition
offers one such example. The proof is provided in the appendix.
\begin{prop}
\label{prop:stability_grad_descent_counterex}
There exists an unstable system $\phi_w$ where gradient descent does not
converge to a stationary point, and $\norm{\grad_w p} \to\infty$ as the number
of iterations $N \to\infty$.
\end{prop}

\subsection{Examples of Stable Recurrent Models}
\label{sec:examples}
In this section, we provide sufficient conditions to ensure stability for
several common recurrent models. These conditions offer a way to require
learning happens in the stable regime-- after each iteration of gradient
descent, one imposes the corresponding stability condition via
projection.



\paragraph{Linear dynamical systems and recurrent neural networks.}
Given a Lipschitz, point-wise non-linearity $\rho$ and matrices
$W \in \reals^{n \times n}$ and $U \in \reals^{n \times d}$, 
the state-transition map for a recurrent neural
network (RNN) is
\begin{align*}
    h_t = \rho(W h_{t-1} + U x_t).
\end{align*}
If $\rho$ is the identity, then the system is a linear dynamical system.
\cite{jin1994absolute} show if $\rho$ is $L_\rho$-Lipschitz, then the model is
stable provided $\norm{W} < \frac{1}{L_\rho}$. Indeed, for any states $h, h'$,
and any $x$,
\begin{align*}
    \norm{\rho(W h + U x) - \rho(W h' + Ux)}
    \leq L_\rho\norm{W h + U x - W h' - Ux}
    \leq L_\rho\norm{W}\norm{h - h'}.
\end{align*}
In the case of a linear dynamical system, the model is stable provided $\norm{W}
< 1$. Similarly, for the 1-Lipschitz $\tanh$-nonlinearity, stability obtains
provided $\norm{W} < 1$. In the appendix, we verify the assumptions
required by the theorems given in the next section for this example.
Imposing this condition during training corresponds to projecting onto the
spectral norm ball.


\paragraph{Long short-term memory networks.}
Long Short-Term Memory (LSTM) networks are another commonly used class of sequence
models \cite{hochreiter1997long}. The state is a pair of vectors 
$s = (c, h) \in \reals^{2d}$, and the model is parameterized by eight
matrices, $W_\square \in \reals^{d \times d}$ and $U_\square \in \reals^{d \times n}$,
for $\square \in \set{i, f, o, z}$. The state-transition map
$\phi_{\mathrm{LSTM}}$ is given by
\begin{align*}
    f_t &= \sigma(W_f h_{t-1} + U_f x_t) \\
    i_t &= \sigma(W_i h_{t-1} + U_i x_t) \\
    o_t &= \sigma(W_o h_{t-1} + U_o x_t) \\
    z_t &= \tanh(W_z h_{t-1} + U_z x_t) \\
    c_t &= i_t \circ z_t + f_t \circ c_{t-1}  \\
    h_t &= o_t \cdot \tanh(c_t),
\end{align*}
where $\circ$ denotes elementwise multiplication, and $\sigma$ is the logistic
function. 

We provide conditions under which the iterated system $\phi_{\mathrm{LSTM}}^r =
\phi_{\mathrm{LSTM}} \circ \cdots \circ \phi_{\mathrm{LSTM}}$ is stable. Let
$\norm{f}_\infty = \sup_t \norm{f_t}_\infty$. If the weights $W_f, U_f$ and
inputs $x_t$ are bounded, then  $\norm{f}_\infty < 1$ since $\abs{\sigma} <
1$ for any finite input. This means the next state $c_t$ must ``forget'' a
non-trivial portion of $c_{t-1}$. We leverage this phenomenon to give sufficient
conditions for $\phi_{\mathrm{LSTM}}$ to be contractive in the $\ell_\infty$
norm, which in turn implies the iterated system $\phi_{\mathrm{LSTM}}^r$ is
contractive in the $\ell_2$ norm for $r=O(\log(d))$. 
Let $\norm{W}_\infty$ denote the induced $\ell_\infty$ matrix norm, which
corresponds to the maximum absolute row sum~$\max_i\sum_j|W_{ij}|.$
\begin{prop}
\label{prop:lstm-stability}
If $\norm{W_i}_\infty, \norm{W_o}_\infty < \paren{1 - \norm{f}_\infty}$,
$\norm{W_z}_\infty \leq (1/4)(1 - \norm{f}_\infty)$, $\norm{W_f}_\infty <
(1-\norm{f}_\infty)^2$, and $r = O(\log(d))$, then the iterated system
$\phi_{\mathrm{LSTM}}^r$ is stable.
\end{prop}

The proof is given in the appendix. 
The conditions given in Proposition~\ref{prop:lstm-stability} are fairly
restrictive.  Somewhat surprisingly we show in the experiments models
satisfying these stability conditions still achieve good performance on a number
of tasks.  We leave it as an open problem to find different parameter regimes
where the system is stable, as well as resolve whether the original system
$\phi_{\mathrm{LSTM}}$ is stable. Imposing these conditions during training and
corresponds to simple row-wise normalization of the weight matrices and inputs.
More details are provided in Section~\ref{sec:experiments} and the appendix.

%
%

\section{Stable Recurrent Models Have Feed-forward Approximations}
\label{sec:ffnn_approximation}
In this section, we prove stable recurrent models can be well-approximated by
feed-forward networks for the purposes of both inference and training by
gradient descent. From a memory perspective, stable recurrent models are
\emph{equivalent} to feed-forward networks---both models use the same amount of
context to make predictions.
This equivalence has important consequences for sequence modeling in practice.
When a stable recurrent model achieves satisfactory performance on some task, a
feed-forward network can achieve similar performance. Consequently, if sequence
learning in practice is inherently stable, then recurrent models may not be
necessary. Conversely, if feed-forward models cannot match the
performance of recurrent models, then sequence learning in practice is in the
unstable regime.

\subsection{Truncated recurrent models}
For our purposes, the salient distinction between a recurrent and feed-forward
model is the latter has \emph{finite-context}. Therefore, we say a model is
\emph{feed-forward} if the prediction made by the model at step $t$ is a
function only of the inputs $x_{t-k}, \dots, x_t$ for some finite $k$.

While there are many choices for a feed-forward approximation, we consider the
simplest one---truncation of the system to some finite context $k$. In other
words, the feed-forward approximation moves over the input sequence with a
sliding window of length~$k$ producing an output every time the sliding window
advances by one step.  Formally, for context length~$k$ chosen in advance, we
define the \emph{truncated model} via the update rule
\begin{align}
    \hk_t = \phi_w(\hk_{t-1}, x_t), \quad \hk_{t-k} = 0\,. \label{eq:truncated-system}
\end{align}
Note that $\hk_t$ is a function only of the previous~$k$ inputs $x_{t-k}, \dots,
x_t$. 
While this definition is perhaps an abuse of the term ``feed-forward'', the
truncated model can be implemented as a standard autoregressive, depth-$k$
feed-forward network, albeit with significant weight sharing. 

Let $f$ denote a prediction function that maps a state~$h_t$ to outputs~$f(h_t)
= y_t$. Let $\yk_t$ denote the predictions from the truncated model. To simplify
the presentation, the prediction function $f$ is not parameterized.  This is
without loss of generality because it is always possible to fold the parameters
into the system $\phi_w$ itself. In the sequel, we study $\norm{y_t - \yk_t}$
both during and after training.

\subsection{Approximation during inference}
Suppose we train a full recurrent model $\phi_w$ and obtain a prediction $y_t$.
For an appropriate choice of context $k$, the truncated model makes essentially 
the same prediction $\yk_t$ as the full recurrent model. To show this result, we
first control the difference between the hidden states of both models.
\begin{lemma}
    \label{lem:trunc-hidden}
    Assume $\phi_w$ is $\lambda$-contractive in $h$ and $L_x$-Lipschitz in $x$. Assume
    the input sequence $\norm{x_t} \leq B_x$ for all $t$. If the truncation
    length $k \geq \log_{1/\lambda}\paren{\frac{L_x B_x}{(1-\lambda)\eps}}$, then the
    difference in hidden states $\norm{h_t - \hk_t} \leq \eps$.
\end{lemma}
Lemma~\ref{lem:trunc-hidden} effectively says stable models do not have
long-term memory-- distant inputs do not change the states of the
system. A proof is given in the appendix.
If the prediction
function is Lipschitz, Lemma~\ref{lem:trunc-hidden} immediately implies
the recurrent and truncated model make nearly identical predictions.
\begin{prop}
    \label{prop:inference}
    If~$\phi_w$ is a $L_x$-Lipschitz and~$\lambda$-contractive map, and~$f$ is $L_f$
    Lipschitz, and the truncation length
    $k \geq \log_{1/\lambda} \paren{\frac{L_f L_x B_x}{(1-\lambda)\eps}}$,
    then $\norm{y_t - \yk_t} \leq \eps$.
\end{prop}


\subsection{Approximation during training via gradient descent}
\label{sec:grad_descent}
Equipped with our inference result, we turn towards optimization.  We show
gradient descent for stable recurrent models finds essentially the same
solutions as gradient descent for truncated models. Consequently, both the
recurrent and truncated models found by gradient descent make essentially the
same predictions. 

Our proof technique is to initialize both the recurrent and
truncated models at the same point and track the divergence in weights
throughout the course of gradient descent. Roughly, we show if $k \approx
O(\log(N/\eps))$, then after $N$ steps of gradient descent, the difference in
the weights between the recurrent and truncated models is at most $\eps$. 
Even if the gradients are similar for both models at the same point, it is a
priori possible that slight differences in the gradients accumulate over time
and lead to divergent weights where no meaningful comparison is possible.
Building on similar techniques as~\cite{hardt2016train}, we show that gradient
descent itself is stable, and this type of divergence cannot occur.

Our gradient descent result requires two essential lemmas. The first bounds the
difference in gradient between the full and the truncated model. The second
establishes the gradient map of both the full and truncated models is Lipschitz.
We defer proofs of both lemmas to the appendix.

Let $p_T$ denote the loss function evaluated on recurrent model after $T$ time
steps, and define $\pk_T$ similarly for the truncated model. Assume there some
compact, convex domain~$\domain \subset \reals^m$ so that the map~$\phi_w$ is
stable for all choices of parameters $w \in \domain$.

\begin{lemma}\label{lem:truncation}
Assume $p$ (and therefore $\pk$) is Lipschitz and smooth.
Assume $\phi_w$ is smooth, $\lambda$-contractive, and Lipschitz in $x$
and $w$. Assume the inputs satisfy
$\norm{x_t} \leq B_x$, then
\begin{align*}
    \norm{\grad_w p_T - \grad_w \pk_T} = \gamma k\lambda^k,
\end{align*}
where $\gamma = O\paren{B_x (1-\lambda)^{-2}}$, suppressing dependence
on the Lipschitz and smoothness parameters.
\end{lemma}

\begin{lemma}
    \label{lem:smoothness}
    For any $w, w' \in \domain$, suppose $\phi_w$ is smooth, $\lambda$-contractive,
    and Lipschitz in $w$. If $p$ is Lipschitz and smooth, then
    \begin{align*}
        \norm{\grad_w p_T(w) - \grad_w p_T(w')}
        \leq \beta\norm{w - w'},
    \end{align*}
    where $\beta = O\paren{(1-\lambda)^{-3}}$, suppressing dependence on
    the Lipschitz and smoothness parameters.
\end{lemma}

Let $\rw^i$ be the weights of the recurrent model on step $i$ and
define $\tw^{i}$ similarly for the truncated model. At initialization, $\rw^0 =
\tw^0$. For $k$ sufficiently large, Lemma~\ref{lem:truncation} guarantees the
difference between the gradient of the recurrent and truncated models is
negligible. Therefore, after a gradient update, $\norm{\rw^1 - \tw^1}$ is small.
Lemma~\ref{lem:smoothness} then guarantees that this small difference in
weights does not lead to large differences in the gradient on the subsequent
time step. For an appropriate choice of learning rate, formalizing
this argument leads to the following proposition.
\begin{prop}
    \label{prop:grad_descent_bound}
    Under the assumptions of Lemmas~\ref{lem:truncation}
    and~\ref{lem:smoothness}, for compact, convex $\domain$, after $N$ steps of
    projected gradient descent with step size $\alpha_t = \alpha / t$,
    $\norm{\rw^N - \tw^N} \leq \alpha \gamma k \lambda^k N^{\alpha \beta + 1}$.
\end{prop}
The decaying step size in our theorem is consistent with the regime in which
gradient descent is known to be stable for non-convex training
objectives~\cite{hardt2016train}. While the decay is faster than many learning
rates encountered in practice, classical results nonetheless show that with this
learning rate gradient descent still converges to a stationary point; see p.~119
in~\cite{bertsekas99nonlinear} and references there. In
the appendix,  we give empirical evidence the $O(1/t)$ rate
is necessary for our theorem and show examples of stable systems trained with
constant or $O(1/\sqrt{t})$ rates that do not satisfy our bound.

Critically, the bound in Proposition~\ref{prop:grad_descent_bound} goes to 0
as $k \to\infty$. In particular, if we take $\alpha = 1$ and $k \geq
\Omega(\log(\gamma N^\beta / \eps))$, then after $N$ steps of projected
gradient descent, $\norm{\rw^N - \tw^N} \leq \eps$. For this choice of $k$, we
obtain the main theorem. The proof is left to the appendix.
\begin{thm}
    \label{thm:main-formal}
    Let $p$ be Lipschitz and smooth. Assume $\phi_w$ is smooth,
    $\lambda$-contractive, Lipschitz in $x$ and $w$. Assume the inputs are
    bounded, and the prediction function $f$ is $L_f$-Lipschitz.  If $k \geq
    \Omega(\log(\gamma N^\beta / \eps))$, then after
    $N$ steps of projected gradient descent with step size $\alpha_t = 1/t$,
    $\norm{y_T - \yk_T} \leq \eps$.
\end{thm}

\section{Experiments}
\label{sec:experiments}
In the experiments, we show stable recurrent models can achieve solid
performance on several benchmark sequence tasks. Namely, we show
unstable recurrent models can often be made stable without a loss in
performance. In some cases, there is a small gap between the performance between
unstable and stable models. We analyze whether this gap is
indicative of a ``price of stability'' and show the unstable models involved are
stable in a data-dependent sense.

\subsection{Tasks}
We consider four benchmark sequence problems--word-level language modeling,
character-level language modeling, polyphonic music modeling, and slot-filling.

\paragraph{Language modeling.}
In language modeling, given a sequence of words or characters, the model must
predict the next word or character.  For character-level language modeling, we
train and evaluate models on Penn Treebank \cite{marcus1993building}. To
increase the coverage of our experiments, we train and evaluate the word-level
language models on the Wikitext-2 dataset, which is twice as large as Penn
Treebank and features a larger vocabulary \cite{merity2016pointer}. Performance
is reported using bits-per-character for character-level models and perplexity
for word-level models.

\paragraph{Polyphonic music modeling.}
In polyphonic music modeling, a piece is represented as a sequence of 88-bit
binary codes corresponding to the 88 keys on a piano, with a 1 indicating a key
that is pressed at a given time. Given a sequence of codes, the task is to
predict the next code. We evaluate our models on JSB Chorales, a polyphonic
music dataset consisting of 382 harmonized chorales by J.S. Bach
\cite{allan2005harmonising}.  Performance is measured using negative
log-likelihood.

\paragraph{Slot-filling.}
In slot filling, the model takes as input a query like ``I want to Boston on
Monday'' and outputs a class label for each word in the input,  e.g. Boston maps
to {\tt Departure\_City} and Monday maps to {\tt Departure\_Time}. We use the
Airline Travel Information Systems (ATIS) benchmark and report the F1 score for
each model \cite{price1990evaluation}.

\subsection{Comparing Stable and Unstable Models}
For each task, we first train an unconstrained RNN and an unconstrained LSTM. All the
hyperparameters are chosen via grid-search to maximize the performance of the
unconstrained model. For consistency with our theoretical results in
Section~\ref{sec:ffnn_approximation} and stability conditions in
Section~\ref{sec:examples}, both models have a single recurrent layer and are
trained using plain SGD.  In each case, the resulting model is unstable.
However, we then retrain the best models using projected gradient descent to
enforce stability \emph{without retuning the hyperparameters}. In the RNN case,
we constrain $\norm{W} < 1$. After each gradient update, we project the
$W$ onto the spectral norm ball by computing the SVD and thresholding the
singular values to lie in $[0, 1)$. In the LSTM case, after each gradient
update, we normalize each row of the weight matrices to satisfy the
sufficient conditions for stability given in Section~\ref{sec:examples}.
Further details are given in the appendix.

\paragraph{Stable and unstable models achieve similar performance.}
Table~\ref{table:stable_results} gives a comparison of the performance between
stable and unstable RNNs and LSTMs on each of the different tasks. Each of the
reported metrics is computed on the held-out test set. We also show a
representative comparison of learning curves for word-level language modeling 
and polyphonic music modeling in Figures~\ref{fig:word_lm_train}
and~\ref{fig:poly_music_train}.

\begin{table}[t]
\caption{
    Comparison of stable and unstable models on a variety of sequence modeling tasks.
    For all the tasks, stable and unstable RNNs achieve the same performance.
    For polyphonic music and slot-filling, stable and unstable LSTMs achieve the
    same results. On language modeling, there is a small gap between stable and
    unstable LSTMs. We discuss this in Section~\ref{sec:lstm_gap}.
    Performance is evaluated on the held-out test set. For negative
    log-likelihood (nll), bits per character (bpc), and perplexity, lower is
    better. For F1 score, higher is better.}
\label{table:stable_results}
\begin{center}
\begin{tabular}{llllll}
&&\multicolumn{4}{c}{\bf Model} \\
&&\multicolumn{2}{c}{RNN} &\multicolumn{2}{c}{LSTM} \\
\multicolumn{1}{c}{\bf Sequence Task} & \multicolumn{1}{c}{\textbf{Dataset} (measure)} & Unstable & Stable & Unstable & Stable \\ 
\hline \\
    Polyphonic Music & JSB Chorales (nll)            &8.9 & 8.9 &
    8.5 & 8.5\\
    Slot-Filling &Atis (F1 score)            &94.7 & 94.7 & 95.1 & 94.6 \\
    Word-level LM &Wikitext-2 (perplexity)         &146.7 & 143.5
    & 95.7 & 113.2 \\
    Character-level LM &Penn Treebank (bpc)
    &1.8 & 1.9 & 1.4 & 1.9
\end{tabular}
\end{center}
\end{table}

\begin{figure}[h!]
    \centering
    \subfigure[Word-level language modeling]{%
        \label{fig:word_lm_train}
        \includegraphics[width=0.45\textwidth]{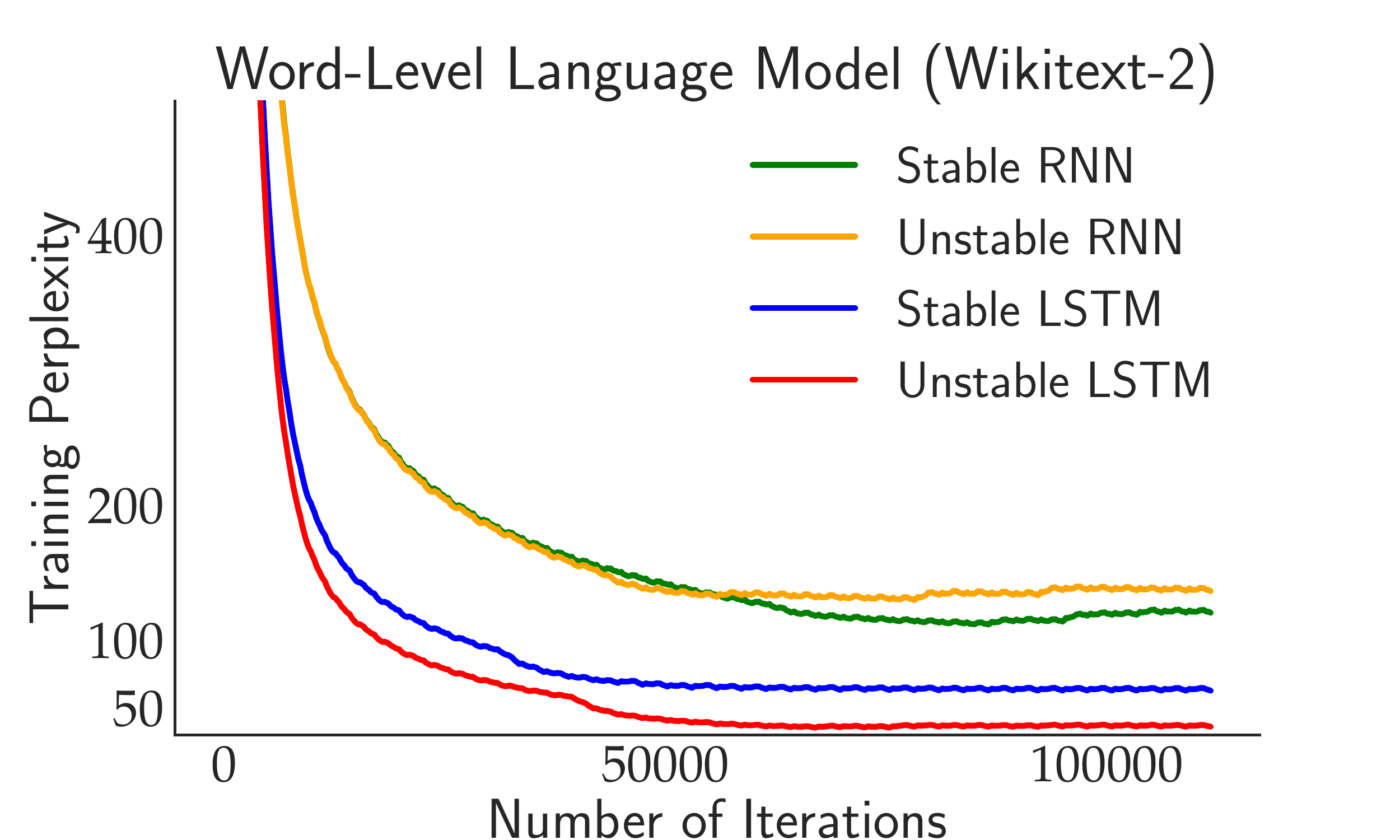}
    }
    \qquad
    \subfigure[Polyphonic music modeling]{%
        \label{fig:poly_music_train}
        \includegraphics[width=0.45\textwidth]{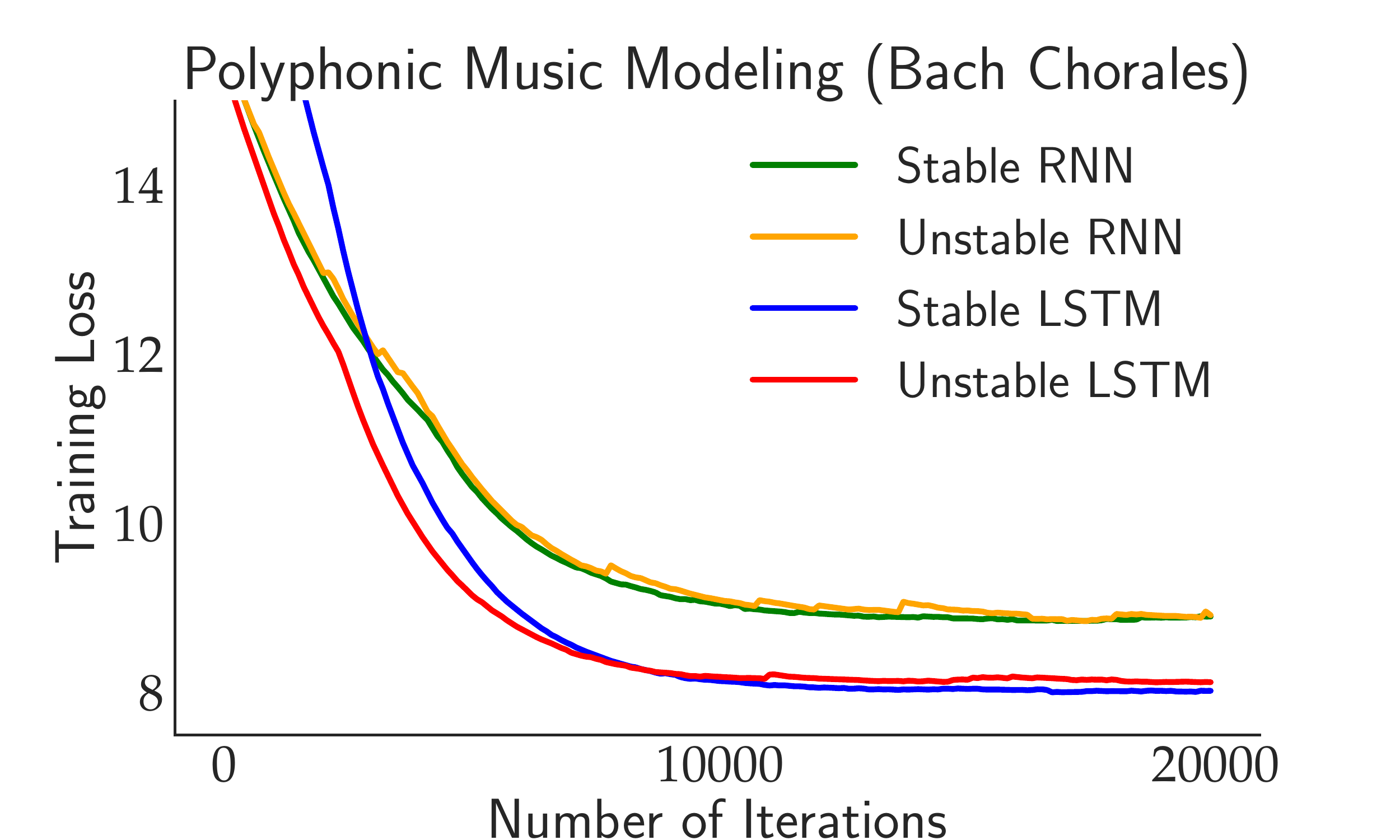}
    }
    \caption{Stable and unstable variants of common recurrent architectures
    achieve similar performance across a range of different sequence tasks.}
\end{figure}

Across all the tasks we considered, stable and unstable RNNs have roughly the
same performance. Stable RNNs and LSTMs achieve results comparable to published
baselines on slot-filling \cite{mesnil2015using} and polyphonic music modeling
\cite{bai2018empirical}. On word and character level language modeling, both
stable and unstable RNNs achieve comparable results to \cite{bai2018empirical}.

On the language modeling tasks, however, there is a gap between stable and
unstable LSTM models. Given the restrictive conditions we place on the LSTM to
ensure stability, it is surprising they work as well as they do. Weaker
conditions ensuring stability of the LSTM could reduce this gap.  It is
also possible imposing stability comes at a cost in representational capacity
required for some tasks. 

\subsection{What is the ``price of stability'' in sequence modeling?}
\label{sec:lstm_gap}
The gap between stable and unstable LSTMs on language modeling raises the
question of whether there is an intrinsic performance cost for using stable
models on some tasks. If we measure stability in a data-dependent fashion, then
the unstable LSTM language models are stable, indicating this gap is illusory.
However, in some cases with short sequences, instability can offer modeling
benefits.

\paragraph{LSTM language models are stable in a ``data-dependent'' way.}
Our notion of stability is conservative and requires stability to hold for
every input and pair of hidden states. If we instead consider a weaker,
data-dependent notion of stability, the word and character-level LSTM 
models are stable (in the iterated sense of Proposition~\ref{prop:lstm-stability}). In
particular, we compute the stability parameter only \emph{using input sequences
from the data}. Furthermore, we only evaluate stability on hidden states
\emph{reachable via gradient descent}. More precisely, to estimate
$\lambda$, we run gradient ascent to find worst-case hidden states $h, h'$ to
maximize $\frac{\norm{\phi_w(h, x) - \phi_w(h',x)}}{\norm{h - h'}}$. More details
are provided in the appendix. 

The data-dependent definition given above is a useful diagnostic---  when the
sufficient stability conditions fail to hold, the data-dependent condition
addresses whether the model is still operating in the stable regime. Moreover,
when the input representation is fixed during training, our theoretical results
go through without modification when using the data-dependent definition. 

Using the data-dependent measure, in Figure~\ref{fig:data_dependent_stability},
we show the iterated character-level LSTM, $\phi_{\mathrm{LSTM}}^r$, is stable
for $r \approx 80$ iterations. A similar result holds for the word-level
language model for $r \approx 100$. These findings are consistent with
experiments in \cite{laurent2017recurrent} which find LSTM trajectories
converge after approximately 70 steps \emph{only when evaluated on sequences
from the data}.  For language models, the ``price of stability'' is therefore
much smaller than the gap in Table~\ref{table:stable_results} suggests-- even
the ``unstable'' models are operating in the stable regime on the data
distribution.

\paragraph{Unstable systems can offer performance improvements for short-time
horizons.} 
When sequences are short, training unstable models is less difficult because
exploding gradients are less of an issue. In these case, unstable models can
offer performance gains.  To demonstrate this, we train truncated unstable
models on the polyphonic music task for various values of the
truncation parameter $k$.  In Figure~\ref{fig:polymusic_unstable}, we
simultaneously plot the performance of the unstable model and the stability
parameter $\lambda$ for the converged model for each $k$. For short-sequences,
the final model is more unstable ($\lambda \approx 3.5$) and achieves a better
test-likelihood.  For longer sequence lengths, $\lambda$ decreases closer to the
stable regime ($\lambda \approx 1.5)$, and this improved test-likelihood
performance disappears.

\begin{figure}[h!]%
    \centering
    \subfigure[Data-dependent stability of character-level language models.
               The iterated-LSTM refers to the iteration system
               $\phi^r_{\mathrm{LSTM}} = \phi_{\mathrm{LSTM}} \circ \cdots \circ \phi_{\mathrm{LSTM}}$.]{%
        \label{fig:data_dependent_stability}
        \includegraphics[width=0.45\textwidth]{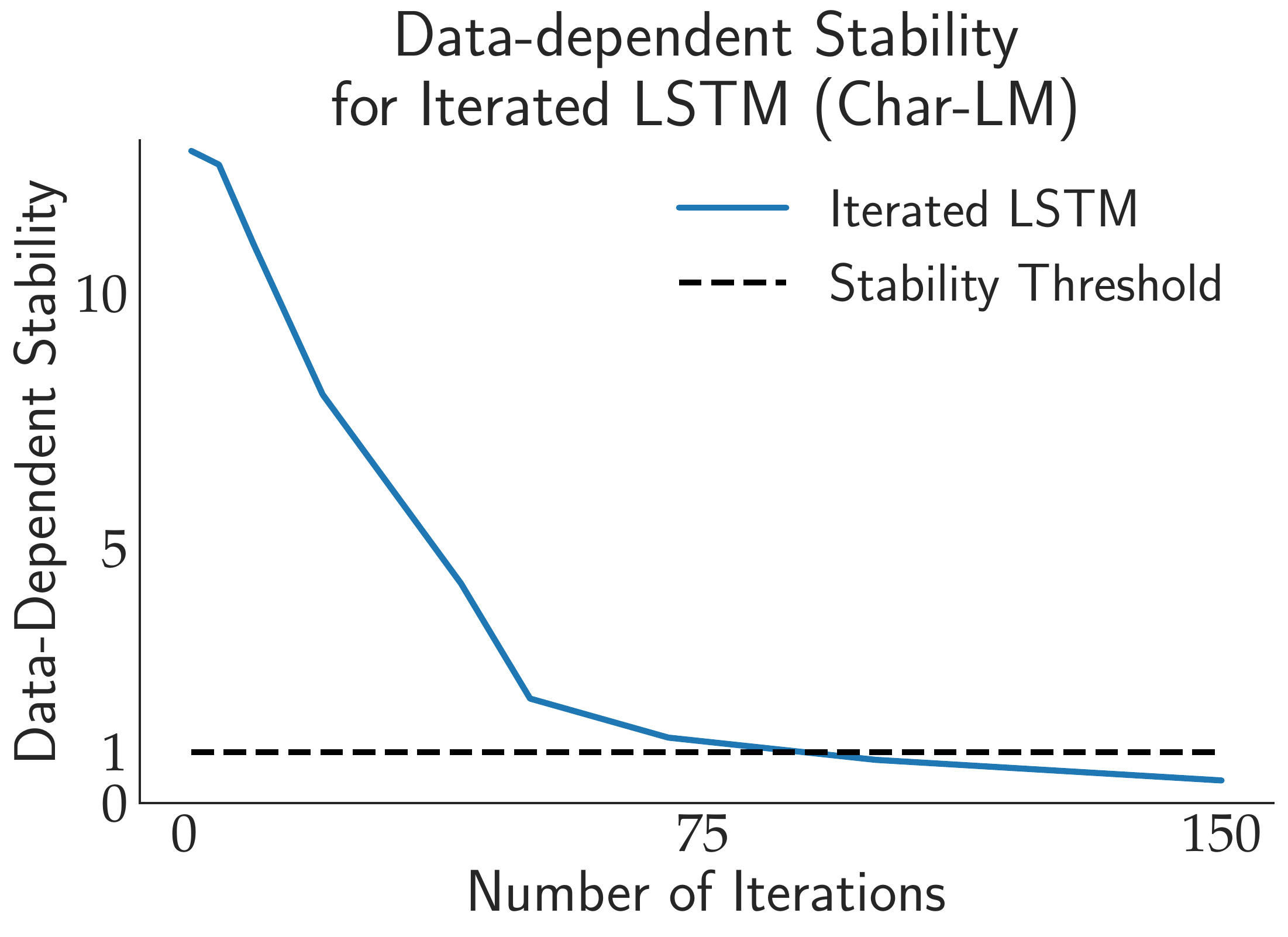}
    }
    \qquad
    \subfigure[Unstable models can boost performance for short sequences.]{%
        \label{fig:polymusic_unstable}
        \includegraphics[width=0.45\textwidth]{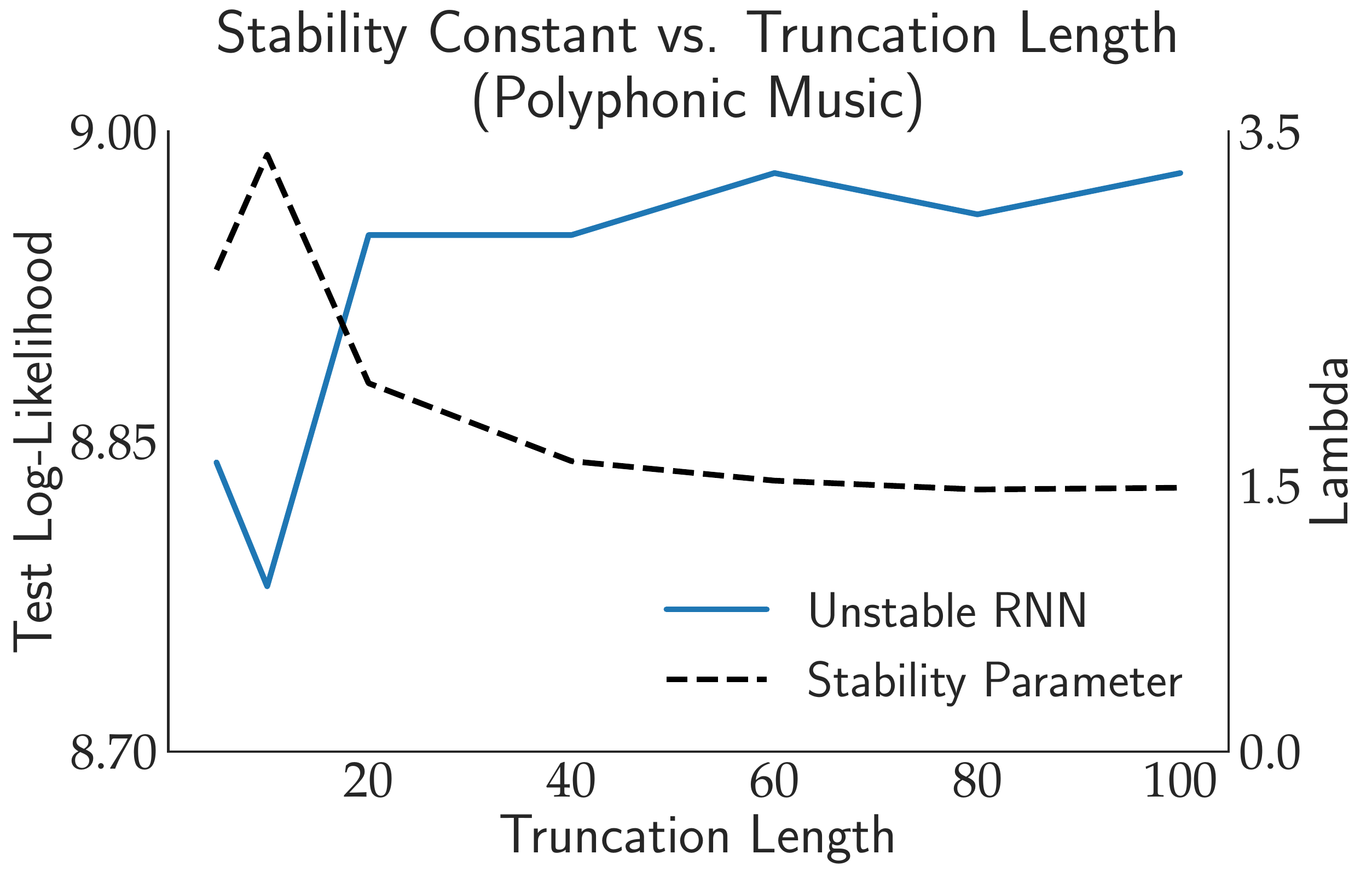}
    }
    \caption{What is the intrinsic ``price of stability''? For language
        modeling, we show the unstable LSTMs are actually stable in weaker,
        data-dependent sense. On the other hand, for polyphonic music modeling
        with short sequences, instability can improve model performance.}
\end{figure}

\subsection{Unstable Models Operate in the Stable Regime}
In the previous section, we showed nominally unstable models often
satisfy a data-dependent notion of stability. In this section, we offer further
evidence unstable models are operating in the stable regime. These results
further help explain why stable and unstable models perform comparably in
experiments.

\paragraph{Vanishing gradients.}
Stable models necessarily have vanishing gradients, and indeed this ingredient
is a key ingredient in the proof of our training-time approximation result. For
both word and character-level language models, we find both \emph{unstable RNNs and
LSTMs also exhibit vanishing gradients}. In
Figures~\ref{fig:word_vanishing}~and~\ref{fig:char_vanishing}, we plot the
average gradient of the loss at time $t+i$ with respect to the input at time
$t$, $\norm{\grad_{x_t} p_{t+i}}$ as $t$ ranges over the training set.  
For either language modeling task, the LSTM and the RNN suffer
from limited sensitivity to distant inputs at initialization and throughout
training. The gradients of the LSTM vanish more slowly than those of the RNN,
but both models exhibit the same qualitative behavior. 
\begin{figure}[h!]%
    \centering
    \subfigure[Word-Level language modeling]{%
        \label{fig:word_vanishing}
        \includegraphics[width=0.45\textwidth]{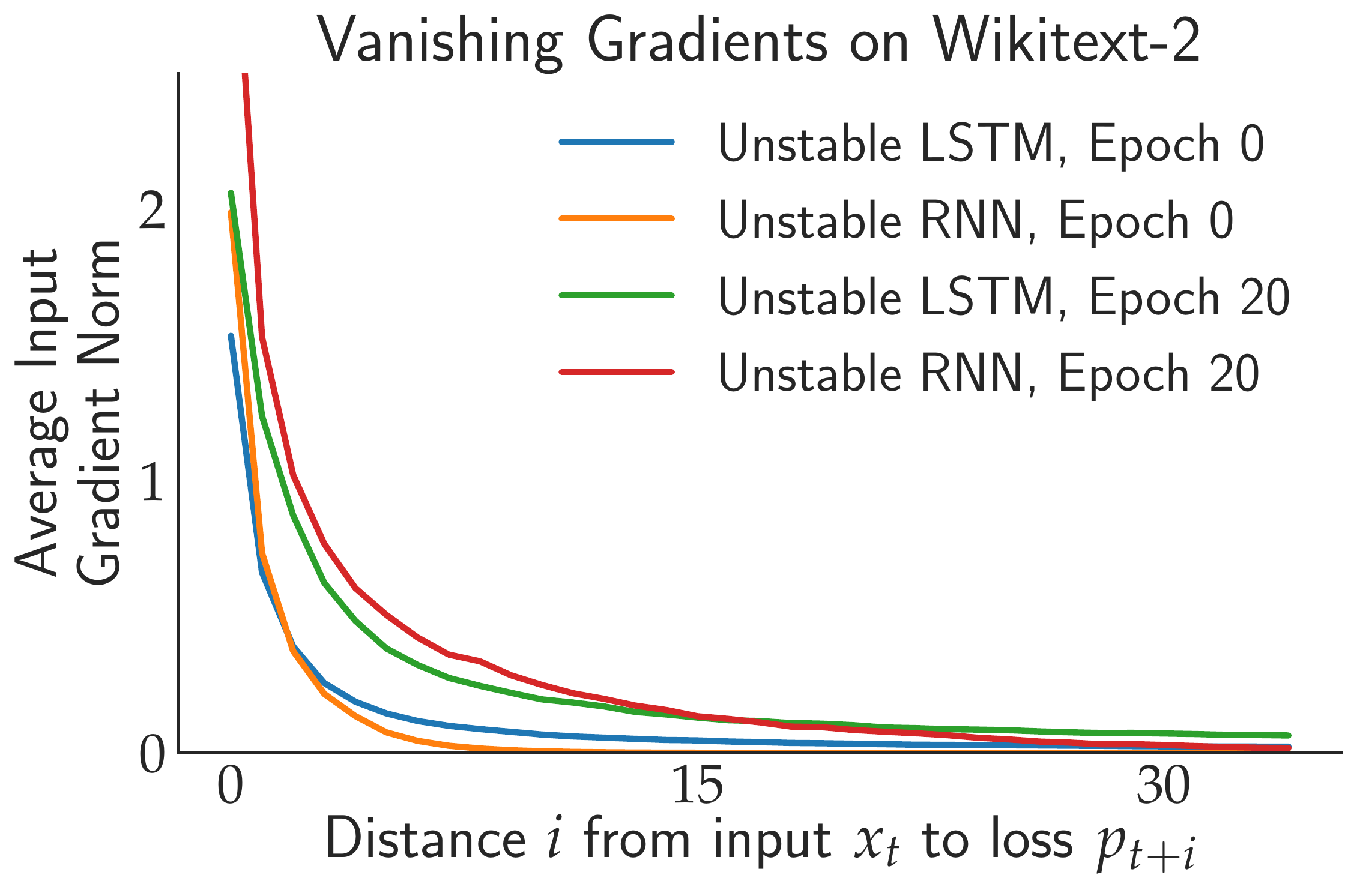}
    }
    \qquad
    \subfigure[Character-level language modeling]{%
        \label{fig:char_vanishing}
        \includegraphics[width=0.45\textwidth]{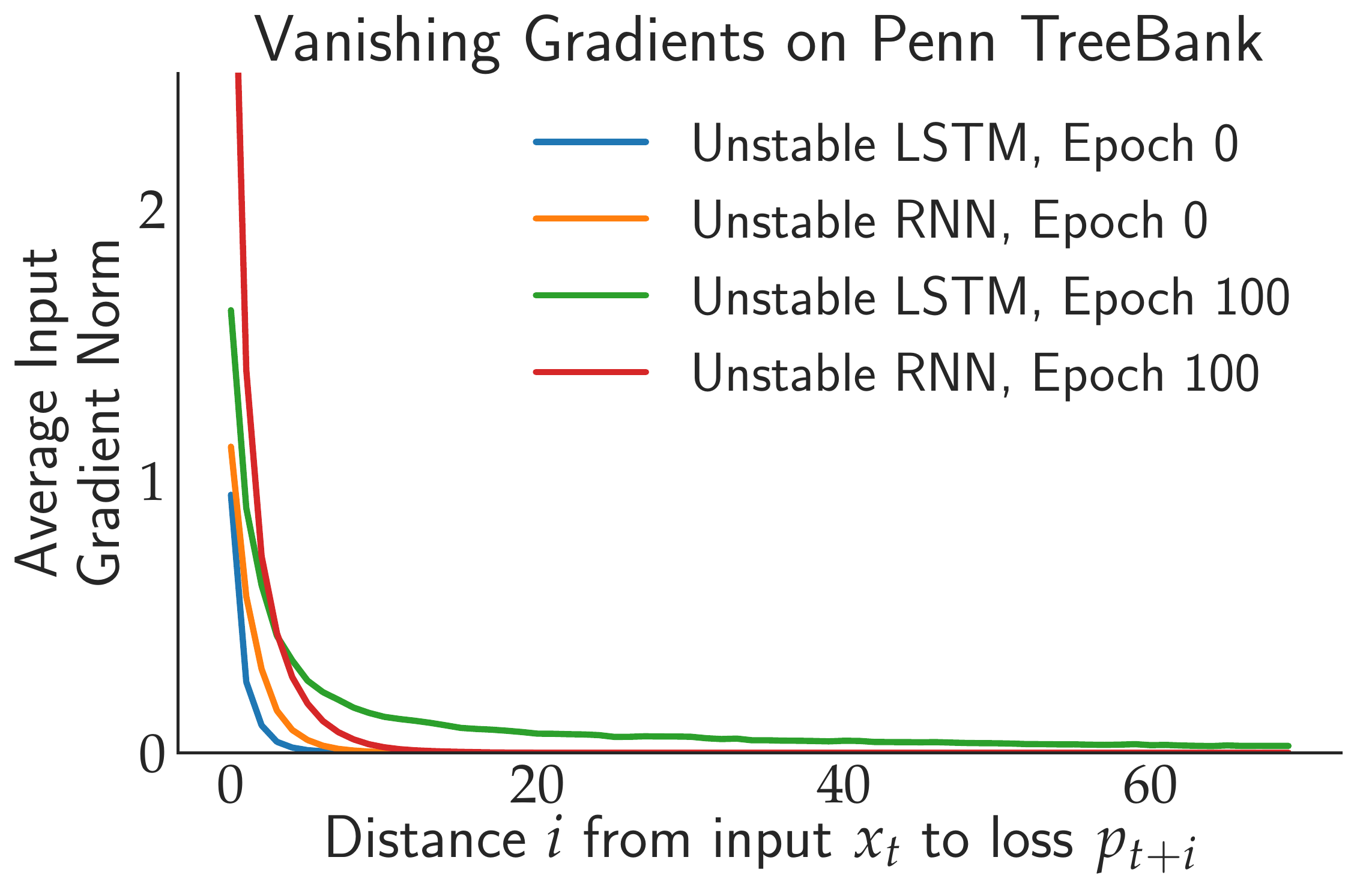}
    }
    \caption{Unstable word and character-level language models exhibit vanishing
             gradients. We plot the norm of the gradient with respect to inputs,
             $\norm{\grad_{x_t} {p_{t+i}}}$, as the distance between the input and
             the loss grows, averaged over the entire training set. The gradient
             vanishes for moderate values of $i$ for both RNNs and LSTMs,
             though the decay is slower for LSTMs.}
\end{figure}

\paragraph{Truncating Unstable Models.} 
The results in Section~\ref{sec:ffnn_approximation} show stable models can be
truncated without loss of performance. In practice, unstable models
can also be truncated without performance loss.
In Figures~\ref{fig:word_lm_trunc}~and~\ref{fig:poly_music_trunc}, we show the
performance of both LSTMs and RNNs for various values of the truncation
parameter $k$ on word-level language modeling and polyphonic music modeling.
Initially, increasing $k$ increases performance because the model can use more
context to make predictions.  However, in both cases, there is diminishing
returns to larger values of the truncation parameter $k$. LSTMs are
unaffected by longer truncation lengths, whereas the performance of RNNs
slightly degrades as $k$ becomes very large, possibly due to training
instability.  In either case, diminishing returns to performance for large
values of $k$ means truncation and therefore feed-forward approximation is
possible even for these unstable models.

\begin{figure}[h!]%
    \centering
    \subfigure[Word-level language modeling]{%
        \label{fig:word_lm_trunc}
        \includegraphics[width=0.45\textwidth]{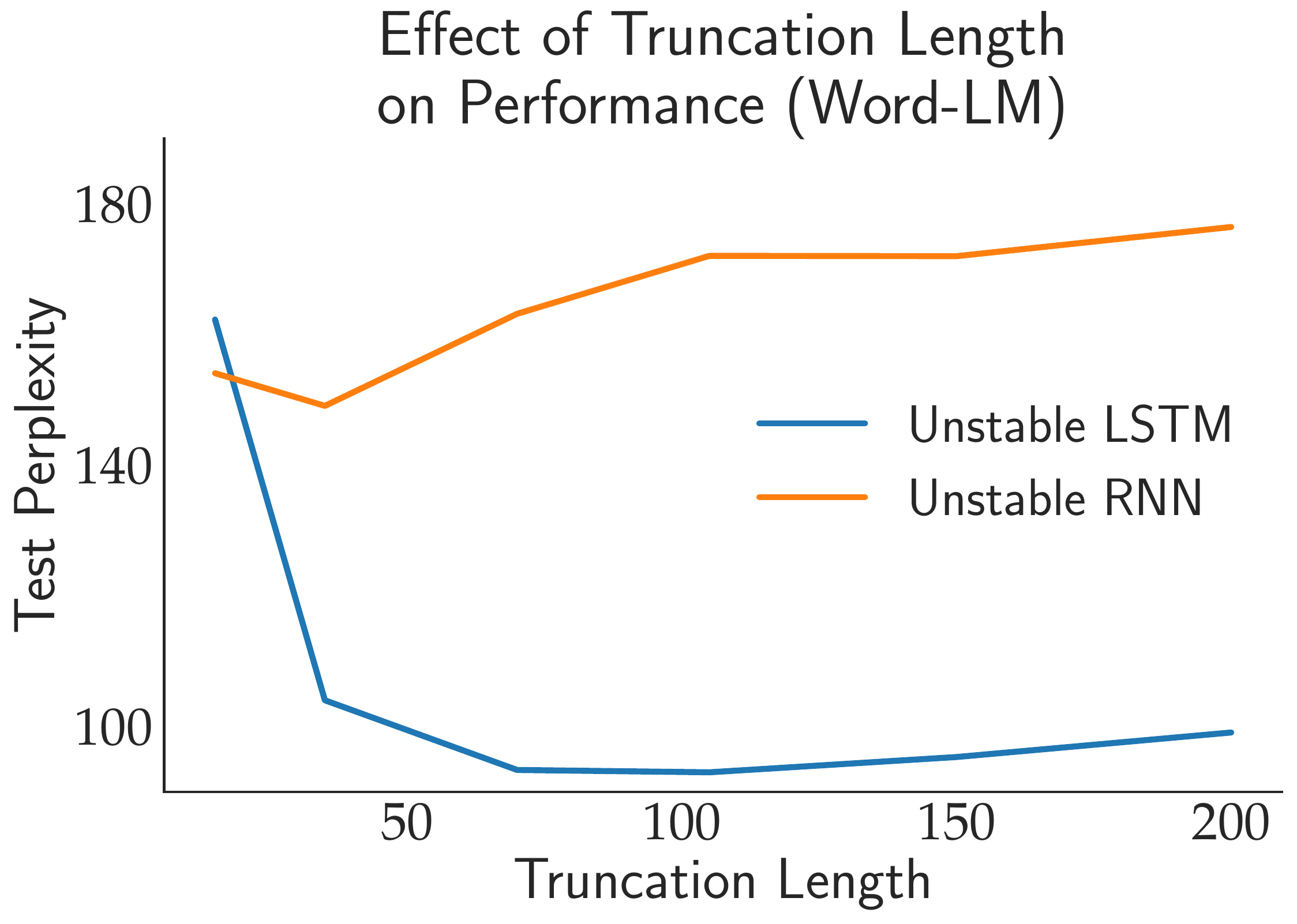}
    }
    \qquad
    \subfigure[Polyphonic music modeling]{%
        \label{fig:poly_music_trunc}
        \includegraphics[width=0.45\textwidth]{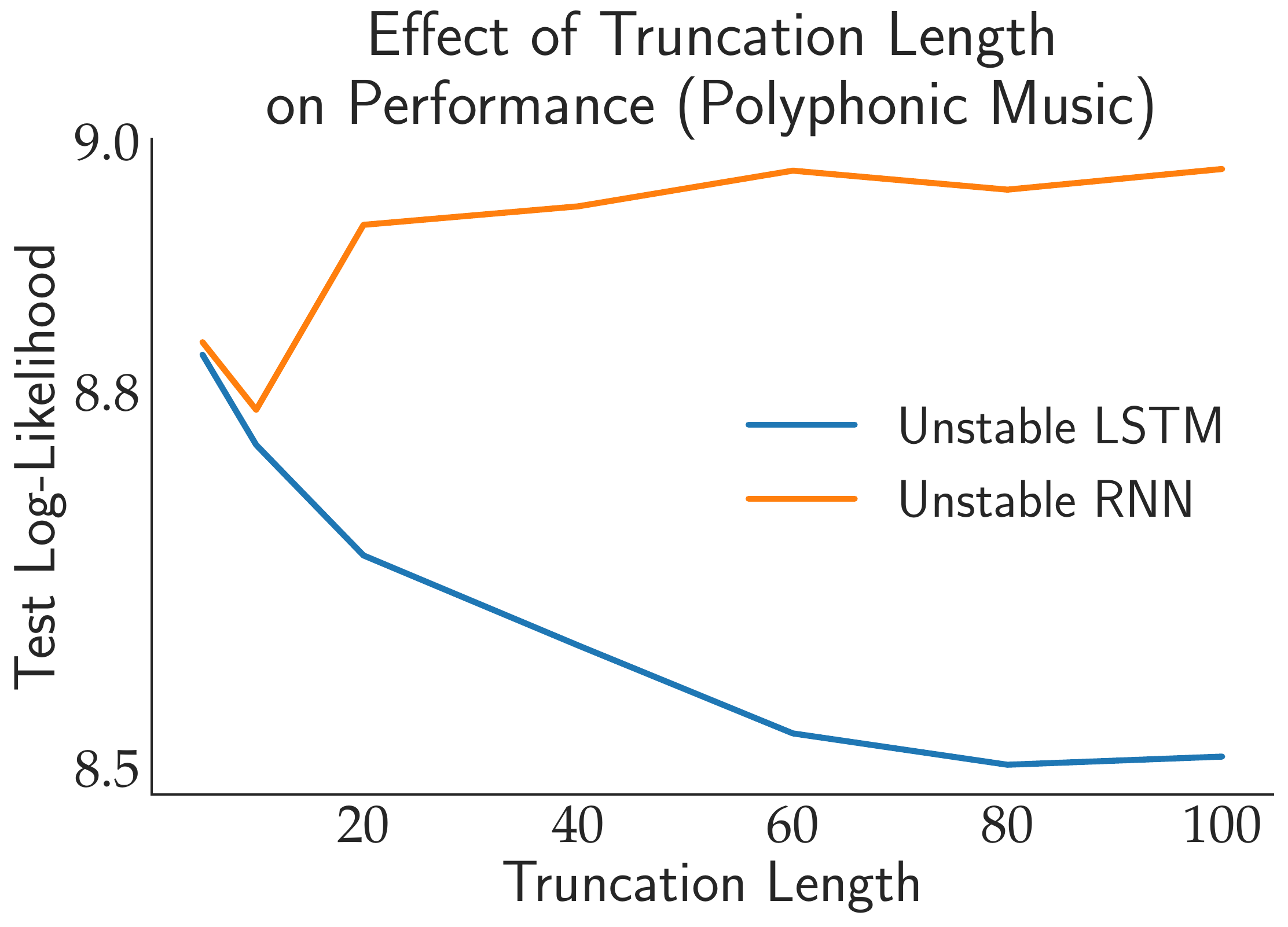}
    }
    \caption{Effect of truncating unstable models. On both language and music
        modeling, RNNs and LSTMs exhibit diminishing returns for large values of
        the truncation parameter $k$. In LSTMs, larger $k$ doesn't affect
        performance, whereas for unstable RNNs, large $k$ slightly decreases
        performance}
\end{figure}

\paragraph{Proposition~\eqref{prop:grad_descent_bound} holds for unstable
models.} In stable models, Proposition~\eqref{prop:grad_descent_bound} in
Section~\ref{sec:ffnn_approximation} ensures the distance between the weight
matrices $\norm{\rw - \tw}$ grows slowly as training progresses, and this rate
decreases as $k$ becomes large. In
Figures~\ref{fig:rnn_divergence}~and~\ref{fig:lstm_divergence}, we show a
similar result holds empirically for unstable word-level language models.  All
the models are initialized at the same point, and we track the distance
between the hidden-to-hidden matrices $W$ as training progresses. Training  the
full recurrent model is impractical, and we assume $k = 65$ well captures the
full-recurrent model. In
Figures~\ref{fig:rnn_divergence}~and~\ref{fig:lstm_divergence}, we plot
$\norm{W_k - W_{65}}$ for $k \in \set{5, 10, 15, 25, 35, 50, 64}$ throughout
training. As suggested by Proposition~\eqref{prop:grad_descent_bound}, after an
initial rapid increase in distance, $\norm{W_k -W_{65}}$ grows slowly, as
suggested by Proposition~\ref{prop:grad_descent_bound}.  Moreover, there is a
diminishing return to choosing larger values of the truncation parameter $k$ in
terms of the accuracy of the approximation.

\begin{figure}[h!]%
    \centering
    \subfigure[Unstable RNN language model]{%
        \label{fig:rnn_divergence}
        \includegraphics[width=0.45\textwidth]{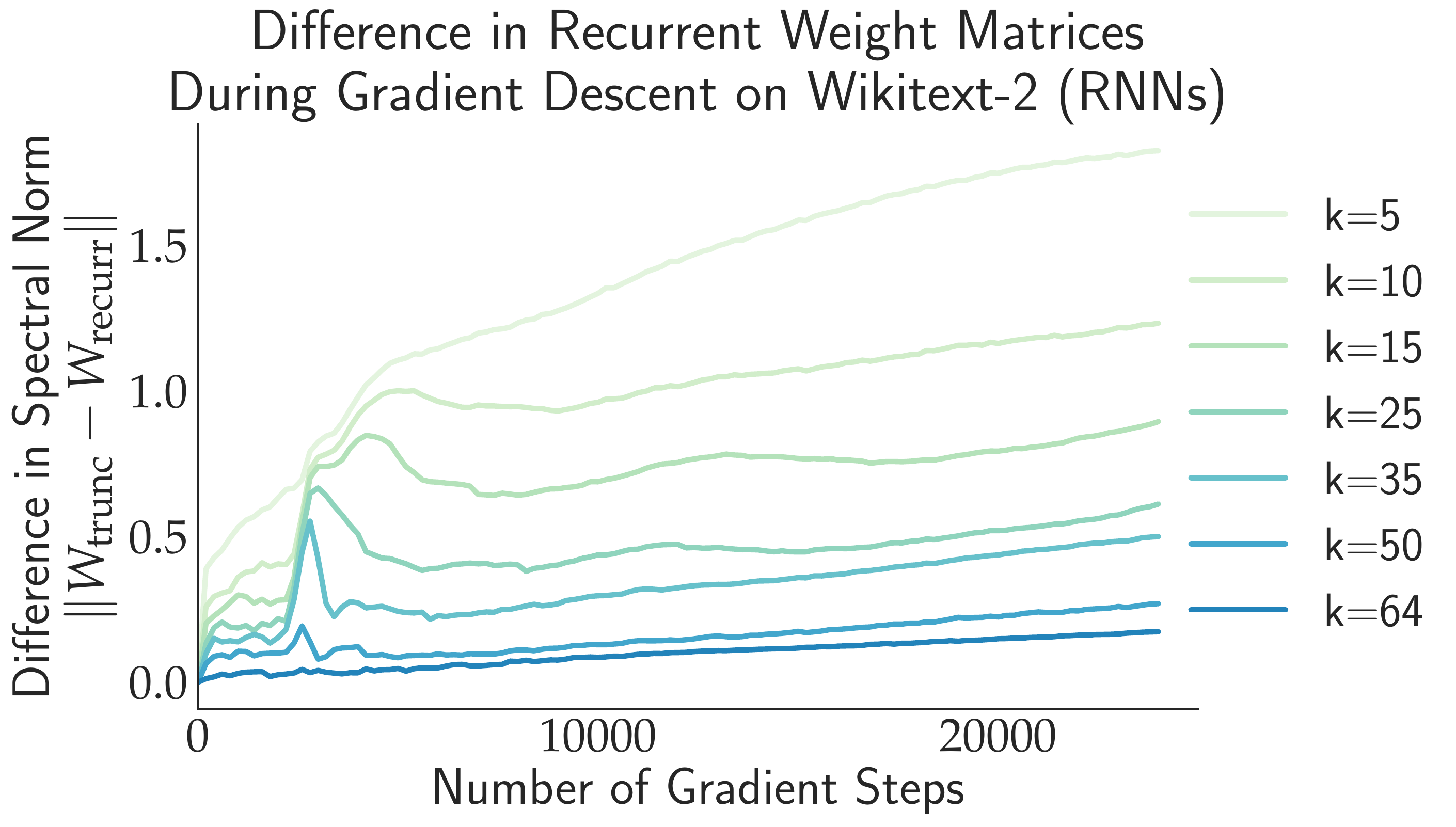}
    }
    \qquad
    \subfigure[Unstable LSTM language model]{%
        \label{fig:lstm_divergence}
        \includegraphics[width=0.45\textwidth]{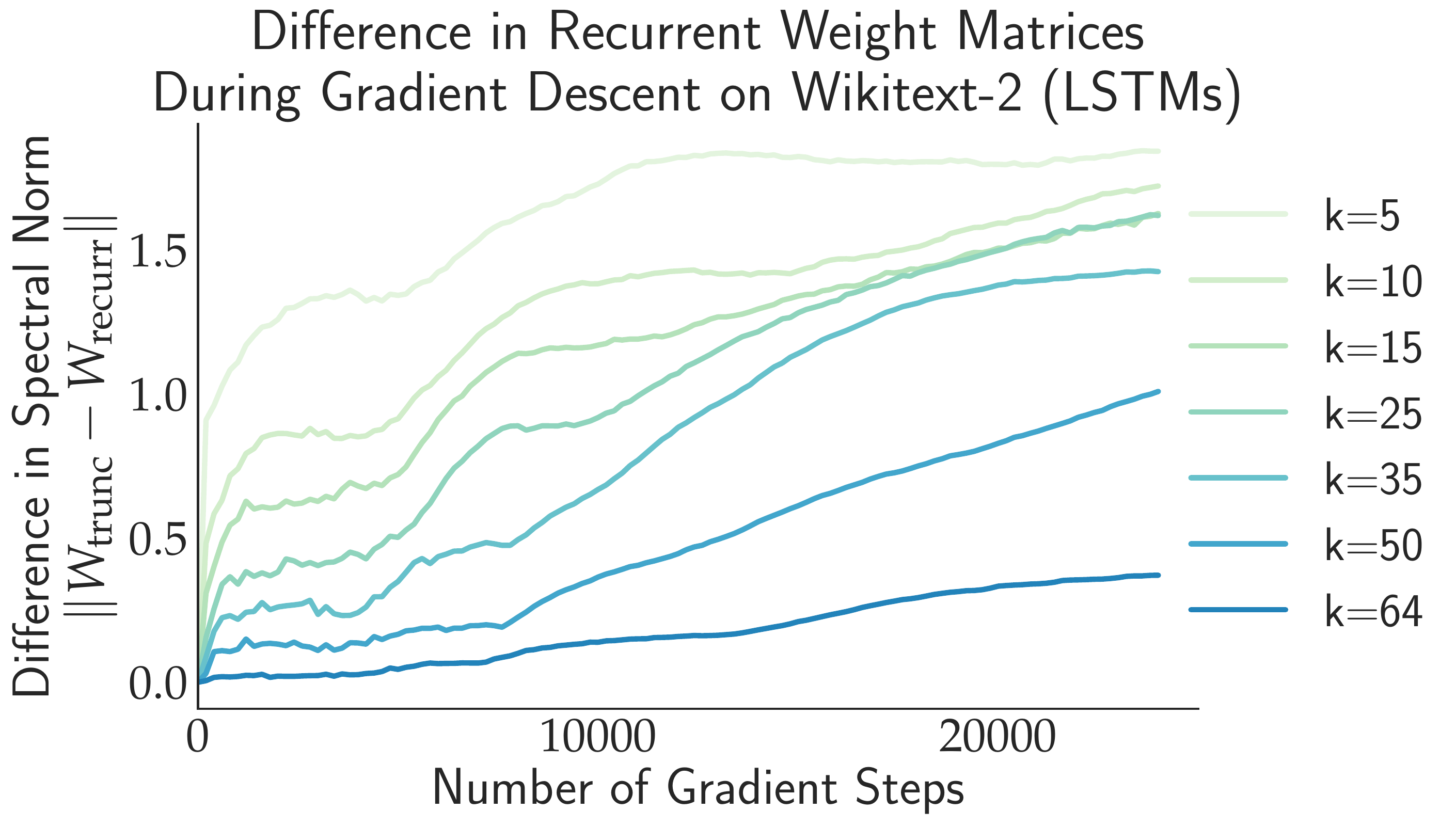}
    }
    \caption{Qualitative version of Proposition~\ref{prop:grad_descent_bound} for
        unstable, word-level language models. We assume $k=65$ well-captures the
        full-recurrent model and plot $\norm{\tw - \rw} = \norm{W_k - W_{65}}$ as training
        proceeds, where $W$ denotes the recurrent weights. As
        Proposition~\ref{prop:grad_descent_bound} suggests, this quantity
        grows slowly as training proceeds, and the rate of growth decreases as
        $k$ increases.}
\end{figure}

\section{Are recurrent models truly necessary?}
\label{sec:discussion}
Our experiments show recurrent models trained in practice operate in the
stable regime, and our theoretical results show stable recurrent models are
approximable by feed-forward networks,  As a consequence, we conjecture
\emph{recurrent networks trained in practice are always approximable by
feed-forward networks}. Even with this conjecture, we cannot yet conclude
recurrent models as commonly conceived are unnecessary. First, our present
proof techniques rely on truncated versions of recurrent models, and truncated
recurrent architectures like LSTMs may provide useful inductive bias on some
problems. Moreover, implementing the truncated approximation as a feed-forward
network increases the number of weights by a factor of $k$ over the original
recurrent model. Declaring recurrent models truly superfluous would require
both finding more parsimonious feed-forward approximations and proving
natural feed-forward models, e.g. fully connected networks or CNNs, can
approximate stable recurrent models during training. This remains an important
question for future work.

\section{Related Work}
\label{sec:related_work}
Learning dynamical systems with gradient descent has been a recent topic of
interest in the machine learning community.  \cite{hardt2018gradient} show
gradient descent can efficiently learn a class of stable, linear dynamical
systems, \cite{oymak2018stochastic} shows gradient descent learns a class of
stable, non-linear dynamical systems. Work by~\cite{sedghi2016training} gives a
moment-based approach for learning some classes of stable non-linear recurrent
neural networks. Our work explores the theoretical and empirical consequences of
the stability assumption made in these works. In particular, our empirical
results show models trained in practice can be made closer to those currently
being analyzed theoretically without large performance penalties.

For \emph{linear} dynamical systems, \cite{tu2017non} exploit the
connection between stability and truncation to learn a truncated
approximation to the full stable system.  Their approximation result is the same
as our inference result for linear dynamical systems, and we extend this
result to the non-linear setting. We also analyze the impact of truncation on
training with gradient descent.  Our training time analysis builds on the
stability analysis of gradient descent in~\cite{hardt2016train}, but
interestingly uses it for an entirely different purpose.  Results of this kind
are completely new to our knowledge.  

For RNNs, the link between vanishing and exploding gradients and
$\norm{W}$ was identified in \cite{pascanu2013difficulty}.  For 1-layer RNNs,
\cite{jin1994absolute} give sufficient conditions for stability in terms of the
norm $\norm{W}$ and the Lipschitz constant of the non-linearity. Our work
additionally considers LSTMs and provides new sufficient conditions for
stability. Moreover, we study the consequences of stability in terms of
feed-forward approximation.

A number of recent works have sought to avoid vanishing and exploding gradients
by ensuring the system is an isometry, i.e. $\lambda = 1$. In the RNN case, this
amounts to constraining $\norm{W} = 1$ \cite{arjovsky2016unitary,
wisdom2016full, jing2017tunable, mhammedi2017efficient,jose2018kronecker}.
\cite{vorontsov2017orthogonality} observes strictly requiring $\norm{W} = 1$
reduces performance on several tasks, and instead proposes maintaining $\norm{W}
\in [1-\eps, 1+\eps]$.  \cite{zhang2018stabilizing} maintains this
``soft-isometry'' constraint using a parameterization based on the SVD that
obviates the need for the projection step used in our stable-RNN experiments.
\cite{kusupati2018fastgrnn} sidestep these issues and stabilizes training using a
residual parameterization of the model.  At present, these unitary models have
not yet seen widespread use, and our work shows much of the sequence learning in
practice, even with nominally unstable models, actually occurs in the stable
regime.

From an empirical perspective,~\cite{laurent2017recurrent} introduce a
non-chaotic recurrent architecture and demonstrate it can perform as well more
complex models like LSTMs. \cite{bai2018empirical}
conduct a detailed evaluation of recurrent and convolutional, feed-forward
models on a variety of sequence modeling tasks. In diverse settings, they find
feed-forward models outperform their recurrent counterparts.  Their experiments
are complimentary to ours; we find recurrent models can often be replaced with
stable recurrent models, which we show are equivalent to feed-forward networks.

\section*{Acknowledgements}
This material is based upon work supported by the National Science Foundation
Graduate Research Fellowship Program under Grant No. DGE 1752814 and a
generous grant from the AWS Cloud Credits for Research program.

\bibliographystyle{abbrv}
\bibliography{main}

\newpage
\appendix
\section{Proofs from Section~\ref{sec:stable_rnns}}

\subsection{Gradient descent on unstable systems need not converge}
\begin{proof}[Proof of Proposition~\ref{prop:stability_grad_descent_counterex}]
Consider a scalar linear dynamical system 
\begin{align}
    \label{eq:scalar_lds1}
    h_t &= a h_{t-1} + b x_t \\
    \hat{y}_t &= h_t,
\end{align}
where $h_0 = 0$, $a, b\in \reals$ are parameters, and $x_t, y_t \in \reals$ are
elements the input-output sequence $\set{(x_t, y_t)}_{t=1}^T$, 
where $L$ is the sequence length, and $\hat{y}_t$ is the prediction at time $t$. 
Stability of the above system corresponds to $\abs{a} < 1$. 

Suppose $(x_t, y_t) = (1, 1)$ for $t=1, \dots, L$. Then the desired
system~\eqref{eq:scalar_lds1} simply computes the identity mapping. 
Suppose we use the squared-loss  $\ell(y_t, \hat{y}_t) = (1/2)(y_t -
\hat{y}_t)^2$, and suppose further $b=1$, so the problem reduces to learning $a =
0$. We first compute the gradient. Compactly write
\begin{align*}
    h_t = \sum_{i=0}^{t-1} a^t b = \paren{\frac{1 - a^t}{1-a}}.
\end{align*}
Let $\delta_t = (\hat{y}_t - y_t)$. The gradient for step $T$ is then
\begin{align*}
    \der{}{a}\ell(y_T, \hat{y}_T) 
    = \delta_T \der{h_T}{a} 
    &= \delta_T \sum_{t=0}^{T-1} a^{T-1-t}h_t \\
    &= \delta_T \sum_{t=0}^{T-1} a^{T-1-t}\paren{\frac{1 - a^t}{1-a}} \\
    &= \delta_T \brack{\frac{1}{(1-a)} \sum_{t=0}^{T-1} a^{t} - \frac{Ta^{T-1}}{(1-a)}} \\
    &= \delta_T \brack{\frac{(1-a^T)}{(1-a)^2} - \frac{Ta^{T-1}}{(1-a)}}.
\end{align*}
Plugging in $y_t = 1$, this becomes
\begin{align}
    \label{eq:scalar_grad}
    \der{}{a}\ell(y_T, \hat{y}_T) 
    =\paren{\frac{(1 - a^T)}{(1-a)} - 1}\brack{\frac{(1-a^T)}{(1-a)^2} - \frac{Ta^{T-1}}{(1-a)}}.
\end{align}
For large $T$, if $\abs{a} > 1$, then $a^L$ grows exponentially with $T$ and the gradient is
approximately
\begin{align*}
    \der{}{a}\ell(y_T, \hat{y}_T) 
    \approx \paren{a^{T-1} - 1} Ta^{T-2} \approx T a^{2T - 3}
\end{align*}
Therefore, if $a^0$ is initialized outside of $[-1, 1]$, the iterates
$a^i$ from gradient descent with step size $\alpha_i = (1/i)$ diverge,
i.e. $a^i \to \infty$, and from equation~\eqref{eq:scalar_grad}, it is clear that such
$a^i$ are not stationary points.
\end{proof}

\subsection{Proofs from section~\ref{sec:examples}}
\subsubsection{Recurrent neural networks}
Assume $\norm{W} \leq \lambda < 1$ and $\norm{U} \leq B_U$. Notice
$\tanh'(x) = 1-\tanh(x)^2$, so since $\tanh(x) \in [-1, 1]$, $\tanh(x)$ is
1-Lipschitz and $2$-smooth. We previously showed the system is stable since,
for any states $h, h'$,
\begin{align*}
    &\norm{\tanh(W h + U x) - \tanh(W h' + Ux)} \\
    &\leq \norm{Wh + Ux - W h' - Ux} \\
    &\leq \norm{W}\norm{h - h'}.
\end{align*}
Using  Lemma~\ref{lem:trunc-hidden} with $k=0$, $\norm{h_t} \leq \frac{B_U
B_x}{(1-\lambda)}$ for all $t$. Therefore, for any $W, W'$, $U, U'$, 
\begin{align*}
    &\norm{\tanh(W h_t + U x) - \tanh(W' h_t + U'x)} \\
    &\leq \norm{Wh_t + Ux - W' h_t - U'x} \\
    &\leq \sup_t \norm{h_t} \norm{W - W'} + B_x\norm{U-U'}.\\
    &\leq \frac{B_U B_x}{(1-\lambda)} \norm{W - W'} + B_x\norm{U-U'},
\end{align*}
so the model is Lipschitz in $U, W$. We can similarly argue the model is $B_U$
Lipschitz in $x$. For smoothness, the partial derivative with respect to $h$ is 
\begin{align*}
    \pder{\phi_w(h,x)}{h} = \diag(\tanh'(Wh + Ux)) W,
\end{align*}
so for any $h, h'$, bounding the $\ell_\infty$ norm with the $\ell_2$ norm,
\begin{align*}
    \norm{\pder{\phi_w(h,x)}{h} - \pder{\phi_w(h',x)}{h}}
    &= \norm{\diag(\tanh'(Wh + Ux)) W - \diag(\tanh'(Wh' + Ux)) W} \\
    &\leq \norm{W} \norm{\diag(\tanh'(Wh + Ux) - \tanh'(Wh' + Ux))} \\
    &\leq 2\norm{W} \norm{Wh + Ux - Wh' - Ux}_\infty \\
    &\leq 2\lambda^2 \norm{h - h'}.
\end{align*}
For any $W, W', U, U'$ satisfying our assumptions,
\begin{align*}
    \norm{\pder{\phi_{w}(h,x)}{h} - \pder{\phi_{w'}(h,x)}{h}}
    &= \norm{\diag(\tanh'(Wh + Ux)) W - \diag(\tanh'(W'h + U'x)) W'} \\
    &\leq \norm{\diag(\tanh'(Wh + Ux) - \tanh'(W'h + U'x))}\norm{W} \\
    &\quad + \norm{\diag(\tanh'(W'h + U'x))}\norm{W - W'} \\
    &\leq 2\lambda \norm{(W-W')h + (U - U')x}_\infty + \norm{W - W'}  \\
    &\leq 2\lambda \norm{(W-W')} \norm{h} + 2\lambda\norm{U - U'}\norm{x} + \norm{W - W'} \\
    &\leq \frac{2\lambda B_U B_x + (1-\lambda)}{(1-\lambda)} \norm{W - W'} + 2\lambda B_x \norm{U - U'}.
\end{align*}
Similar manipulations establish $\pder{\phi_w(h, x)}{w}$ is Lipschitz in $h$ and $w$.

\subsubsection{LSTMs}
Similar to the previous sections, we assume $s_0 = 0$.

The state-transition map is not Lipschitz in $s$, much less stable, unless
$\norm{c}$ is bounded. However, assuming the weights are bounded, we first prove
this is always the case.

\begin{lemma}
\label{lem:lstm-no-blow-up}
Let $\norm{f}_\infty = \sup_t \norm{f_t}_\infty$. 
If
$\norm{W_f}_\infty < \infty$,
$\norm{U_f}_\infty < \infty$, 
and 
$\norm{x_t}_\infty \leq B_x$,
then
$\norm{f}_\infty < 1$
and
$\norm{c_t}_\infty \leq \frac{1}{(1-\norm{f}_\infty)}$
for all $t$.
\end{lemma}
\begin{proof}[Proof of Lemma~\ref{lem:lstm-no-blow-up}]
Note $\abs{\tanh(x)}, \abs{\sigma(x)} \leq 1$ for all $x$. Therefore, for any
$t$, $\norm{h_t}_\infty = \norm{o_t \circ \tanh(c_t)}_\infty \leq 1$. Since
$\sigma(x) < 1$ for $x < \infty$ and $\sigma$ is monotonically increasing
\begin{align*}
    \norm{f_t}_\infty 
    &\leq \sigma\paren{\norm{W_f h_{t-1} + U_f x_t}_\infty} \\
    &\leq \sigma\paren{\norm{W_f}_\infty \norm{h_{t-1}}_\infty + \norm{U_f}_\infty \norm{x_t}_\infty} \\
    &\leq \sigma\paren{B_W + B_u B_x} \\
    &< 1.
\end{align*}
Using the trivial bound, $\norm{i_t}_\infty \leq 1$ and $\norm{z_t}_\infty \leq
1$, so
\begin{align*}
    \norm{c_{t+1}}_\infty
    = \norm{i_t\circ z_t + f_t \circ c_t}_\infty
    &\leq 1 + \norm{f_t}_\infty \norm{c_t}_\infty.
\end{align*}
Unrolling this recursion, we obtain a geometric series
\begin{align*}
    \norm{c_{t+1}}_\infty
    \leq \sum_{i=0}^{t}   \norm{f_t}_\infty^i
    \leq \frac{1}{(1- \norm{f}_\infty)}.
\end{align*}
\end{proof}

\begin{proof}[Proof of Proposition~\ref{prop:lstm-stability}]
We show $\phi_{\mathrm{LSTM}}$ is $\lambda$-contractive in the
$\ell_\infty$-norm for some $\lambda < 1$. For
$r \geq \log_{1/\lambda}(\sqrt{d})$, this in turn implies the iterated
system $\phi^r_{\mathrm{LSTM}}$ is contractive is the $\ell_2$-norm.

Consider the pair of reachable hidden states $s=(c, h)$, $s' = (c', h')$. By
Lemma~\ref{lem:lstm-no-blow-up}, $c, c'$ are bounded. Analogous to the
recurrent network case above, since $\sigma$ is $(1/4)$-Lipschitz and
$\tanh$ is 1-Lipschitz,
\begin{align*}
    \norm{i - i'} &\leq \frac{1}{4} \norm{W_i}_\infty\norm{h - h'}_\infty \\
    \norm{f - f'} &\leq \frac{1}{4} \norm{W_f}_\infty\norm{h - h'}_\infty \\
    \norm{o - o'} &\leq \frac{1}{4} \norm{W_o}_\infty\norm{h - h'}_\infty \\
    \norm{z - z'} &\leq \norm{W_z}_\infty\norm{h - h'}_\infty.
\end{align*}
Both $\norm{z}_\infty, \norm{i}_\infty \leq 1$ since they're the output of a
sigmoid. Letting $c_+$ and $c_+'$ denote the state on the next time step,
applying the triangle inequality,
\begin{align*}
    \norm{c_+ - c_+'}_\infty
    &\leq \norm{i \circ z - i' \circ z'}_\infty + \norm{f \circ c - f'\circ c'}_\infty \\
    & \leq \norm{(i - i') \circ z}_\infty + \norm{i' \circ (z - z')}_\infty
    + \norm{f \circ (c-c')}_\infty + \norm{c \circ (f-f')}_\infty \\
    &\leq \norm{i - i'}_\infty\norm{z}_\infty + \norm{z - z'}_\infty\norm{i'}_\infty
    + \norm{c-c'}_\infty\norm{f}_\infty + \norm{f-f'}_\infty\norm{c}_\infty \\
    &\leq \paren{\frac{\norm{W_i}_\infty + \norm{c}_\infty \norm{W_f}_\infty}{4} + \norm{W_z}_\infty}
    \norm{h - h'}_\infty + \norm{f}_\infty\norm{c-c'}_\infty.
\end{align*}
A similar argument shows
\begin{align*}
    \norm{h_+ - h_+'}_\infty
    \leq \norm{o - o'}_\infty + \norm{c_+ - c_+'}_\infty
    &\leq \frac{\norm{W_o}_\infty}{4} \norm{h - h'}_\infty 
    +  \norm{c_+ - c_+'}_\infty.
\end{align*}
By assumption,
\begin{align*}
    \paren{\frac{\norm{W_i}_\infty + \norm{c}_\infty \norm{W_f}_\infty + \norm{W_o}_\infty}{4} + \norm{W_z}_\infty} 
    < 1 - \norm{f}_\infty,
\end{align*}
and so
\begin{align*}
    \norm{h_+ - h_+'}_\infty
    < (1 - \norm{f}_\infty)\norm{h - h'}_\infty + \norm{f}_\infty \norm{c - c'}_\infty
    &\leq \norm{s - s'}_\infty,
\end{align*}
as well as
\begin{align*}
    \norm{c_+ - c_+'}_\infty
    < (1 - \norm{f}_\infty)\norm{h - h'}_\infty + \norm{f}_\infty \norm{c - c'}_\infty
    &\leq \norm{s - s'}_\infty,
\end{align*}
which together imply
\begin{align*}
    \norm{s_+ - s'_+}_\infty < \norm{s - s'}_\infty,
\end{align*}
establishing $\phi_{\mathrm{LSTM}}$ is contractive in the $\ell_\infty$
norm.
\end{proof}
\section{Proofs from section~\ref{sec:ffnn_approximation}}

Throughout this section, we assume the initial state $h_0 = 0$. Without loss of
generality, we also assume $\phi_w(0, 0) = 0$ for all $w$.  Otherwise, we can
reparameterize $\phi_w(h, x) \mapsto \phi_w(h, x) - \phi_w(0, 0)$ without
affecting expressivity of $\phi_w$. For stable models, we also assume there some compact, convex
domain~$\domain \subset \reals^m$ so that the map~$\phi_w$ is
stable for all choices of parameters $w \in \domain$.

\begin{proof}[Proof of Lemma~\ref{lem:trunc-hidden}]
For any $t \geq 1$, by triangle inequality,
\begin{align*}
    \norm{h_t}
    = \norm{\phi_w(h_{t-1}, x_t) - \phi_w(0, 0)}
    \leq \norm{\phi_w(h_{t-1}, x_t) - \phi_w(0, x_t)}
     + \norm{\phi_w(0, x_t) - \phi_w(0, 0)}.
\end{align*}
Applying the stability and Lipschitz assumptions and then summing a geometric series,
\begin{align*}
    \norm{h_t}
    \leq \lambda\norm{h_{t-1}} + L_x \norm{x_t}
    \leq \sum_{i=0}^t \lambda^i L_x B_x
    \leq \frac{L_x B_x}{(1-\lambda)}\,.
\end{align*}
Now, consider the difference between hidden states at time step $t$. Unrolling the
iterates $k$ steps and then using the previous display yields
\begin{align*}
    \norm{h_t - \hk_t}
    = \norm{\phi_w(h_{t-1}, x_t) - \phi_w(\hk_{t-1}, x_t)}
    \leq \lambda \norm{h_{t-1} - \hk_{t-1}}
    \leq \lambda^k \norm{h_{t-k}}
    \leq \frac{\lambda^k L_x B_x}{(1-\lambda)},
\end{align*}
and solving for $k$ gives the result.
\end{proof}

%
%
%
\subsection{Proofs from section~\ref{sec:grad_descent}}
Before proceeding, we introduce notation for our smoothness assumption.
We assume the map~$\phi_w$ satisfies four smoothness conditions:
for any reachable states $h, h'$, and any weights $w, w' \in \domain$, there are
some scalars $\beta_{ww}, \beta_{wh}, \beta_{hw}, \beta_{hh}$ such that
\begin{enumerate}
    \item  $\norm{\pder{\phi_w(h, x)}{w} - \pder{\phi_{w'}(h, x)}{w}} \leq \beta_{ww} \norm{w - w'}$.
    \item  $\norm{\pder{\phi_w(h, x)}{w} - \pder{\phi_w(h', x)}{w}} \leq \beta_{wh} \norm{h - h'}$.
    \item  $\norm{\pder{\phi_w(h, x)}{h} - \pder{\phi_{w'}(h, x)}{h}} \leq \beta_{hw} \norm{w - w'}$.
    \item  $\norm{\pder{\phi_w(h, x)}{h} - \pder{\phi_{w}(h', x)}{h}} \leq \beta_{hh} \norm{h - h'}$.
\end{enumerate}

\subsubsection{Gradient difference due to truncation is negligible}
In the section, we argue the difference in gradient with respect to the weights
between the recurrent and truncated models is $O(k \lambda^k)$. For sufficiently
large $k$ (independent of the sequence length), the impact of
truncation is therefore negligible. The proof leverages the
``vanishing-gradient'' phenomenon-- the long-term components of the gradient of
the full recurrent model quickly vanish. The remaining challenge is to show the
short-term components of the gradient are similar for the full and recurrent
models.

\begin{proof}[Proof of Lemma ~\ref{lem:truncation}]
The Jacobian of the loss with respect to the weights is
\begin{align*}
    \pder{p_T}{w}
    = \pder{p_T}{h_T} \paren{\sum_{t=0}^T \pder{h_T}{h_t} \pder{h_t}{w}},
\end{align*}
where $\pder{h_t}{w}$ is the partial derivative of $h_t$ with respect to $w$,
assuming $h_{t-1}$ is constant with respect to $w$. 
Expanding the expression for the gradient, we wish to bound
{\small
\begin{align*}
    \norm{\grad_w p_T(w) - \grad_w \pk_T(w)}
    &=\norm{\sum_{t=1}^T \paren{\pder{h_T}{h_t}\pder{h_t}{w}}^\top \grad_{h_T} p_T -
    \sum_{t=T-k+1}^T \paren{\pder{\hk_T}{\hk_t}\pder{\hk_t}{w}}^\top
    \grad_{\hk_T} \pk_T}\\
    &\leq  \norm{\sum_{t=1}^{T-k} \paren{\pder{h_T}{h_t}
    \pder{h_t}{w}}^\top \grad_{h_T} p_T} \\
    &+ \sum_{t=T-k+1}^T \norm{\paren{\pder{h_T}{h_t} \pder{h_t}{w}}^\top
    \grad_{h_T} p_T - \paren{\pder{\hk_T}{\hk_t} \pder{\hk_t}{w}}^\top
    \grad_{\hk_T} p_T}.
\end{align*}}
The first term consists of the ``long-term components'' of the gradient for the
recurrent model. The second term is the difference in the ``short-term
components'' of the gradients between the recurrent and truncated models.
We bound each of these terms separately.

For the first term, by the Lipschitz assumptions, $\norm{\grad_{h_T} p_T} \leq
L_p$ and $\norm{\grad_w h_t} \leq L_w$. Since $\phi_w$ is
$\lambda$-contractive, so $\norm{\pder{h_t}{h_{t-1}}} \leq \lambda$.
Using submultiplicavity of the spectral norm,
\begin{align*}
    \norm{\pder{p_T}{h_T}\sum_{t=0}^{T-k} \pder{p_T}{h_t}\pder{h_t}{w}}
    \leq \norm{\grad_{h_T} p_T}  \sum_{t=0}^{T-k} \norm{\prod_{i=t}^T
    \pder{h_{i}}{h_{i-1}}}\norm{\grad_w h_t} 
    \leq L_p L_w \sum_{t=0}^{T-k} \lambda^{T-t}
    \leq \lambda^k \frac{L_p L_w}{(1-\lambda)}.
\end{align*}
Focusing on the second term, by triangle inequality and smoothness,
{\small
\begin{align*}
    &\sum_{t=T-k+1}^T \norm{\paren{\pder{h_T}{h_t} 
    \pder{h_t}{w}}^\top \grad_{h_T} p_T - \paren{\pder{\hk_T}{\hk_t} \pder{\hk_t}{w}}^\top \grad_{\hk_T} p_T} \\
    \leq &\sum_{t=T-k+1}^T\norm{\grad_{h_T} p_T - \grad_{\hk_T} \pk_T} \norm{\pder{\hk_T}{\hk_t} \pder{\hk_t}{w}}
     + \norm{\grad_{h_T} p_T} \norm{\pder{h_T}{h_t} \pder{h_t}{w} - \pder{\hk_T}{\hk_t} \pder{\hk_t}{w}} \\
    \leq &\sum_{t=T-k+1}^T \underbrace{\beta_p \norm{h_T - \hk_T}\lambda^{T-t}L_w}_{(a)} + 
     \underbrace{L_p \norm{\pder{h_T}{h_t} \pder{h_t}{w} - \pder{\hk_T}{\hk_t} \pder{\hk_t}{w}}}_{(b)}.
\end{align*}}
Using Lemma~\ref{lem:trunc-hidden} to upper bound (a),
\begin{align*}
    \sum_{t=T-k}^T \beta_p \norm{h_T - \hk_T}\lambda^{T-t}L_w
    \leq \sum_{t=T-k}^T \lambda^{T-t} \frac{\lambda^k \beta_p L_w L_x B_x}{(1-\lambda)}
    \leq\frac{\lambda^k \beta_p L_w L_x B_x}{(1-\lambda)^2}.
\end{align*}
Using the triangle inequality, Lipschitz and smoothness, (b) is bounded by
\begin{align*}
    &\sum_{t=T-k+1}^T L_p \norm{\pder{h_T}{h_t} \pder{h_t}{w} -
    \pder{\hk_T}{\hk_t} \pder{\hk_t}{w}} \\
    &\leq \sum_{t=T-k+1}^T L_p\norm{\pder{h_T}{h_t}}\norm{\pder{h_t}{w} - \pder{\hk_t}{w}}
    + L_p\norm{\pder{\hk_t}{w}}\norm{\pder{h_T}{h_t} - \pder{\hk_T}{\hk_t}} \\
    &\leq \sum_{t=T-k+1}^T L_p \lambda^{T-t}\beta_{wh}\norm{h_t - \hk_t}
    + L_p L_w \norm{\pder{h_T}{h_t} - \pder{\hk_T}{\hk_t}} \\
    &\leq k\lambda^k \frac{L_p \beta_{wh} L_x B_x}{(1-\lambda)} +
    \underbrace{L_p L_w \sum_{t=T-k+1}^T\norm{\pder{h_T}{h_t} - \pder{\hk_T}{\hk_t}}}_{(c)},
\end{align*}
where the last line used 
$\norm{h_t - \hk_t} \leq \lambda^{t-(T-k)} \frac{L_x B_x}{(1-\lambda)}$ 
for $t \geq T-k$.
It remains to bound (c), the difference of the hidden-to-hidden Jacobians.
Peeling off one term at a time and applying triangle inequality, for any
$t \geq T-k+1$,
{\small
\begin{align*}
    \norm{\pder{h_T}{h_t} - \pder{\hk_T}{\hk_t}}
    &\leq\norm{\pder{h_T}{h_{T-1}} - \pder{\hk_T}{\hk_{T-1}}} \norm{\pder{h_{T-1}}{h_{t}}}
    + \norm{\pder{\hk_T}{\hk_{T-1}}} \norm{\pder{h_{T-1}}{h_t} - \pder{\hk_{T-1}}{\hk_t}} \\
    &\leq \beta_{hh}\norm{h_{T-1} - h_{T-1}}\lambda^{T-t-1}
    + \lambda\norm{\pder{h_{T-1}}{h_t} - \pder{\hk_{T-1}}{\hk_t}} \\
    &\leq \sum_{i=t}^{T-1} \beta_{hh} \lambda^{T-t-1} \norm{h_{i} - \hk_{i}} \\
    &\leq  \lambda^k \frac{\beta_{hh} L_x B_x}{(1-\lambda)}\sum_{i=t}^{T-1} \lambda^{i-t} \\
    &\leq \lambda^k \frac{\beta_{hh} L_x B_x}{(1-\lambda)^2},
\end{align*}}
so (c) is bounded by 
$k\lambda^k \frac{L_p L_w \beta_{hh} L_x B_x}{(1-\lambda)^2}$. Ignoring
Lipschitz and smoothness constants, we've shown
the entire sum is $O\paren{\frac{k\lambda^k}{(1-\lambda)^2}}$.
\end{proof}

\subsubsection{Stable recurrent models are smooth}
In this section, we prove that the gradient map $\grad_w p_T$ is Lipschitz.
First, we show on the forward pass, the difference between hidden states
$h_t(w)$ and $\th_t(w')$ obtained by running the model with weights $w$ and
$w'$, respectively, is bounded in terms of $\norm{w - w'}$. Using smoothness of
$\phi$, the difference in gradients can be written in terms of $\norm{h_t(w) -
\th_t(w')}$, which in turn can be bounded in terms of $\norm{w - w'}$. We
repeatedly leverage this fact to conclude the total difference in gradients must
be similarly bounded.

We first show small differences in weights don't significantly change the
trajectory of the recurrent model.
\begin{lemma} 
    \label{lem:smooth-hidden}
    For some $w, w'$, suppose $\phi_w, \phi_{w'}$ are $\lambda$-contractive and
    $L_w$ Lipschitz in $w$.  Let $h_t(w), h_t(w')$ be the hidden state at time
    $t$ obtain from running the model with weights $w, w'$ on common inputs
    $\set{x_t}$. If $h_0(w) = h_0(w')$, then
    \begin{align*}
        \norm{h_t(w) - h_t(w')} \leq \frac{L_w \norm{w - w'} }{(1-\lambda)}.
    \end{align*}
\end{lemma}
\begin{proof}
By triangle inequality, followed by the Lipschitz and contractivity assumptions,
\begin{align*}
    &\norm{h_t(w) - h_t(w')} \\
    &= \norm{\phi_w(h_{t-1}(w), x_t) - \phi_{w'}(h_{t-1}(w'), x_t)} \\
    &\leq \norm{\phi_w(h_{t-1}(w), x_t) - \phi_{w'}(h_{t-1}(w), x_t)}
    + \norm{\phi_{w'}(h_{t-1}(w), x_t) - \phi_{w'}(h_{t-1}(w'), x_t)} \\
    &\leq L_w\norm{w - w'} + \lambda \norm{h_{t-1}(w) - h_{t-1}(w')}.
\end{align*}
Iterating this argument and then using $h_0(w) = h_0(w')$, we obtain a geometric series in $\lambda$.
\begin{align*}
    \norm{h_t(w) - h_t(w')} 
    &\leq L_w\norm{w - w'} + \lambda \norm{h_{t-1}(w) - h_{t-1}(w')} \\
    &\leq \sum_{i=0}^t L_w \norm{w - w'} \lambda^i \\
    &\leq \frac{L_w \norm{w - w'}}{(1-\lambda)}\,.\qedhere
\end{align*}
\end{proof}

The proof of Lemma~\ref{lem:smoothness} is similar in structure to
Lemma~\ref{lem:truncation}, and follows from repeatedly using smoothness of
$\phi$ and Lemma~\ref{lem:smooth-hidden}.
\begin{proof}[Proof of Lemma \ref{lem:smoothness}]
Let $\th_t = h_t(w')$. Expanding the gradients and using
$\norm{h_t(w) - h_t(w')} \leq \frac{L_w \norm{w - w'} }{(1-\lambda)}$
from Lemma~\ref{lem:smooth-hidden}.
\begin{align*}
    &\norm{\grad_w p_{T}(w) - \grad_{w} p_T(w')} \\
    &\leq \sum_{t=1}^T \norm{\paren{\pder{h_T}{h_{t}} \pder{h_{t}}{w}}^\top \grad_{h_T} p_T 
    - \paren{\pder{\th_T}{\th_{t}} \pder{\th_{t}}{w}}^\top \grad_{\th_T} p_T}\\
    &\leq \sum_{t=1}^T \norm{\grad_{h_T} p_T - \grad_{\th_T} p_T}
    \norm{\pder{\th_T}{\th_t} \pder{\th_t}{w}} 
    + \norm{\grad_{h_T} p_T} \norm{\pder{h_T}{h_t} \pder{h_t}{w} - \pder{\th_T}{\th_t} \pder{\th_t}{w}} \\
    &\leq \sum_{t=1}^T \beta_p \norm{h_T - \th_T} \lambda^{T-t} L_w
    + L_p \norm{\pder{h_T}{h_t} \pder{h_t}{w} - \pder{\th_T}{\th_t} \pder{\th_t}{w}} \\
    &\leq \frac{\beta_p L_w^2 \norm{w-w'}}{(1-\lambda)^2} 
    + \underbrace{L_p \sum_{t=1}^T \norm{\pder{h_T}{h_t} \pder{h_t}{w} - \pder{\th_T}{\th_t} \pder{\th_t}{w}}}_{(a)}.
\end{align*}
Focusing on term (a),
\begin{align*}
    &L_p\sum_{t=1}^T \norm{\pder{h_T}{h_t} \pder{h_t}{w} - \pder{\th_T}{\th_t}
    \pder{\th_t}{w}} \\
    &\leq L_p \sum_{t=1}^T \norm{\pder{h_{T}}{h_t}  - \pder{\th_T}{\th_t}}\norm{\pder{h_t}{w}} 
    + L_p\norm{\pder{\th_{T}}{\th_t}} \norm{\pder{h_t}{w} - \pder{\th_t}{w}} \\
    &\leq L_pL_w\sum_{t=1}^T \norm{\pder{h_{T}}{h_t}  - \pder{\th_T}{\th_t}} 
    + L_p \sum_{t=1}^T\lambda^{T-t}\paren{\beta_{wh} \norm{h_t - \th_t} + \beta_{ww}\norm{w - w'}} \\
    &\leq \underbrace{L_pL_w\sum_{t=1}^T \norm{\pder{h_{T}}{h_t}  - \pder{\th_T}{\th_t}}}_{(b)}
    + \frac{L_p \beta_{wh} L_w \norm{w - w'}}{(1-\lambda)^2} + \frac{L_p \beta_{ww} \norm{w - w'}}{(1-\lambda)},
\end{align*}
where the penultimate line used,
\begin{align*}
    \norm{\pder{h_t}{w} - \pder{\th_t}{w}}
    &\leq \norm{\pder{\phi_w(h_{t-1}, x_t)}{w} - \pder{\phi_w(\th_{t-1}, x_t)}{w}}
     + \norm{\pder{\phi_w(\th_{t-1}, x_t)}{w} - \pder{\phi_{w'}(\th_{t-1}, x_t)}{w}} \\
    &\leq \beta_{wh} \norm{h - h'} + \beta_{ww}\norm{w-w'}.
\end{align*}
To bound (b), we peel off terms one by one using the triangle inequality,
\begin{align*}
    &L_p L_w \sum_{t=1}^T \norm{\pder{h_{T}}{h_t}  - \pder{\th_T}{\th_t}} \\
    &\leq L_p L_w\sum_{t=1}^T \norm{\pder{h_T}{h_{T-1}} - \pder{\th_T}{\th_{T-1}}}\norm{\pder{h_{T-1}}{h_t}} 
    + \norm{\pder{\th_{T}}{\th_{T-1}}}\norm{\pder{h_{T-1}}{h_t}  - \pder{\th_{T-1}}{\th_t}} \\
    &\leq L_p L_w \sum_{t=1}^T \brack{\paren{\beta_{hh} \norm{h_{T-1} - \th_{T-1}}
    + \beta_{hw}\norm{w - w'}}\lambda^{T-t-1} + \lambda \norm{\pder{h_{T-1}}{h_t}  - \pder{\th_{T-1}}{\th_t}}} \\
    &\leq L_p L_w \sum_{t=1}^T \brack{\beta_{hw} (T-t)\lambda^{T-t-1}\norm{w-w'}
    + \beta_{hh} \sum_{i=1}^{T-t} \norm{h_{T-i}- \th_{T-i}} \lambda^{T-t-1}} \\
    &\leq L_p L_w \sum_{t=1}^T \brack{\beta_{hw} (T-t)\lambda^{T-t-1}\norm{w-w'}
    + \frac{\beta_{hh}L_w\norm{w-w'}}{(1-\lambda)} (T-t)\lambda^{T-t-1}} \\
    &\leq \frac{L_pL_w \beta_{hw} \norm{w - w'}}{(1-\lambda)^2} 
    + \frac{L_p L_w^2 \beta_{hh}\norm{w - w'}}{(1-\lambda)^3}.
\end{align*}
Supressing Lipschitz and smoothness constants, we've shown
the entire sum is $O(1/(1-\lambda)^{3})$, as required.
\end{proof}

\subsubsection{Gradient descent analysis}
Equipped with the smoothness and truncation lemmas
(Lemmas~\ref{lem:truncation}~and~\ref{lem:smoothness}), we turn
towards proving the main gradient descent result.
\begin{proof}[Proof of Proposition~\ref{prop:grad_descent_bound}]
Let $\Pi_{\domain}$ denote the Euclidean projection onto $\domain$, and let
$\delta_i = \norm{\rw^i - \tw^i}$. Initially $\delta_0 = 0$, and on step
$i+1$, we have the following recurrence relation for $\delta_{i+1}$,
\begin{align*}
    \delta_{i+1}
    &= \norm{\rw^{i+1} - \tw^{i+1}} \\
    &= \norm{\Pi_{\domain}(\rw^i - \alpha_i \grad p_T(w^i)) - \Pi_{\domain}(\tw^i - \alpha_i \grad \pk_T(\tw^i))} \\
    &\leq \norm{\rw^i - \alpha_i \grad p_T(w^i)) - \tw^i - \alpha_i \grad \pk_T(\tw^i)} \\
    &\leq \norm{\rw^i - \tw^i} + \alpha_i \norm{\grad p_T(\rw^i) - \grad \pk_T(\tw^i)} \\
    &\leq \delta_i + \alpha_i \norm{\grad p_T(\rw^i) - \grad p_T(\tw^i)}
    + \alpha_i\norm{\grad p_T(\tw^i) - \grad \pk_T(\tw^i)} \\
    &\leq \delta_i + \alpha_i\paren{\beta \delta_i + \gamma k \lambda^k} \\
    &\leq \exp\paren{\alpha_i \beta}\delta_i + \alpha_i \gamma k \lambda^k,
\end{align*}
the penultimate line applied lemmas~\ref{lem:truncation}
and~\ref{lem:smoothness}, and the last line used $1+x \leq e^x$ for all
$x$. Unwinding the recurrence relation at step $N$,
\begin{align*}
    \delta_N
    &\leq \sum_{i=1}^N \set{\prod_{j=i+1}^N \exp(\alpha_j \beta)}\alpha_i \gamma
    k \lambda^k  \\
    &\leq \sum_{i=1}^N \set{\prod_{j=i+1}^N \exp\paren{\frac{\alpha \beta}{j}}}
    \frac{\alpha \gamma k \lambda^k}{i} \\
    &= \sum_{i=1}^N \set{\exp\paren{\alpha \beta \sum_{j=i+1}^N \frac{1}{j}}} \frac{\alpha \gamma k \lambda^k}{i}.
\end{align*}
Bounding the inner summation via an integral, $\sum_{j=i+1}^N \frac{1}{j}
\leq \log(N/i)$ and simplifying the resulting expression,
\begin{align*}
    \delta_N
    &\leq \sum_{i=1}^N \exp(\alpha \beta \log(N/i))\frac{\alpha \gamma k
    \lambda^k}{i} \\
    &= \alpha \gamma k \lambda^k N^{\alpha \beta} \sum_{i=1}^N
    \frac{1}{i^{\alpha \beta + 1}} \\
    &\leq \alpha \gamma k \lambda^k N^{\alpha \beta + 1}\,.
\end{align*}
\end{proof}

\subsubsection{Proof of theorem~\ref{thm:main-formal}}
\begin{proof}[Proof of Theorem~\ref{thm:main-formal}]
Using $f$ is $L_f$-Lipschitz and the triangle inequality,
\begin{align*}
    \norm{y_T - \yk_T}
    &\leq L_f \norm{h_T(\rw^N) - \hk_T({\tw^N})} \\ 
    &\leq L_f \norm{h_T(\rw^N) -  h_T(\tw^N)} 
    + L_f \norm{h_T(\tw^N) - \hk_T({\tw^N})}.
\end{align*}
By Lemma~\ref{lem:smooth-hidden}, the first term is bounded by $\frac{L_w
\norm{\rw^N - \tw^N}}{(1-\lambda)}$, and by Lemma~\ref{lem:trunc-hidden}, 
the second term is bounded by $\lambda^k \frac{L_x B_x}{(1-\lambda)}$. 
Using Proposition~\ref{prop:grad_descent_bound}, after $N$ steps of gradient
descent, we have
\begin{align*}
    \norm{y_T - \yk_T}
    &\leq \frac{L_f L_w \norm{\rw^N - \tw^N}}{(1-\lambda)}
    + \lambda^k\frac{L_f L_x B_x}{(1-\lambda)}  \\
    &\leq k \lambda^k \frac{\alpha L_f L_w N^{\alpha \beta + 1}}{(1-\lambda)}
    + \lambda^k\frac{L_f L_x B_x}{(1-\lambda)},
\end{align*}
and solving for $k$ such that both terms are less than $\eps/2$ gives the
result.
\end{proof}

\section{Experiments}

\paragraph{The $O(1/t)$ rate may be necessary.}
The key result underlying Theorem~\ref{thm:main-formal} is the bound on the
parameter difference $\norm{\tw - \rw}$ while running gradient descent obtained
in Proposition~\ref{prop:grad_descent_bound}. We show this bound has the correct
qualitative scaling using random instances and training randomly initialized,
stable linear dynamical systems and $\tanh$-RNNs. In
Figure~\ref{fig:synthetic_grad_descent_bounds}, we plot the parameter error
$\norm{\tw^t - \rw^t}$ as training progresses for both models (averaged over 10
runs). The error scales comparably with the bound given in
Proposition~\ref{prop:grad_descent_bound}. We also find for larger step-sizes
like $\alpha/\sqrt{t}$ or constant $\alpha$, the bound fails to hold, suggesting
the $O(1/t)$ condition is necessary.

\begin{figure}%
    \centering \includegraphics[width=0.5\textwidth]{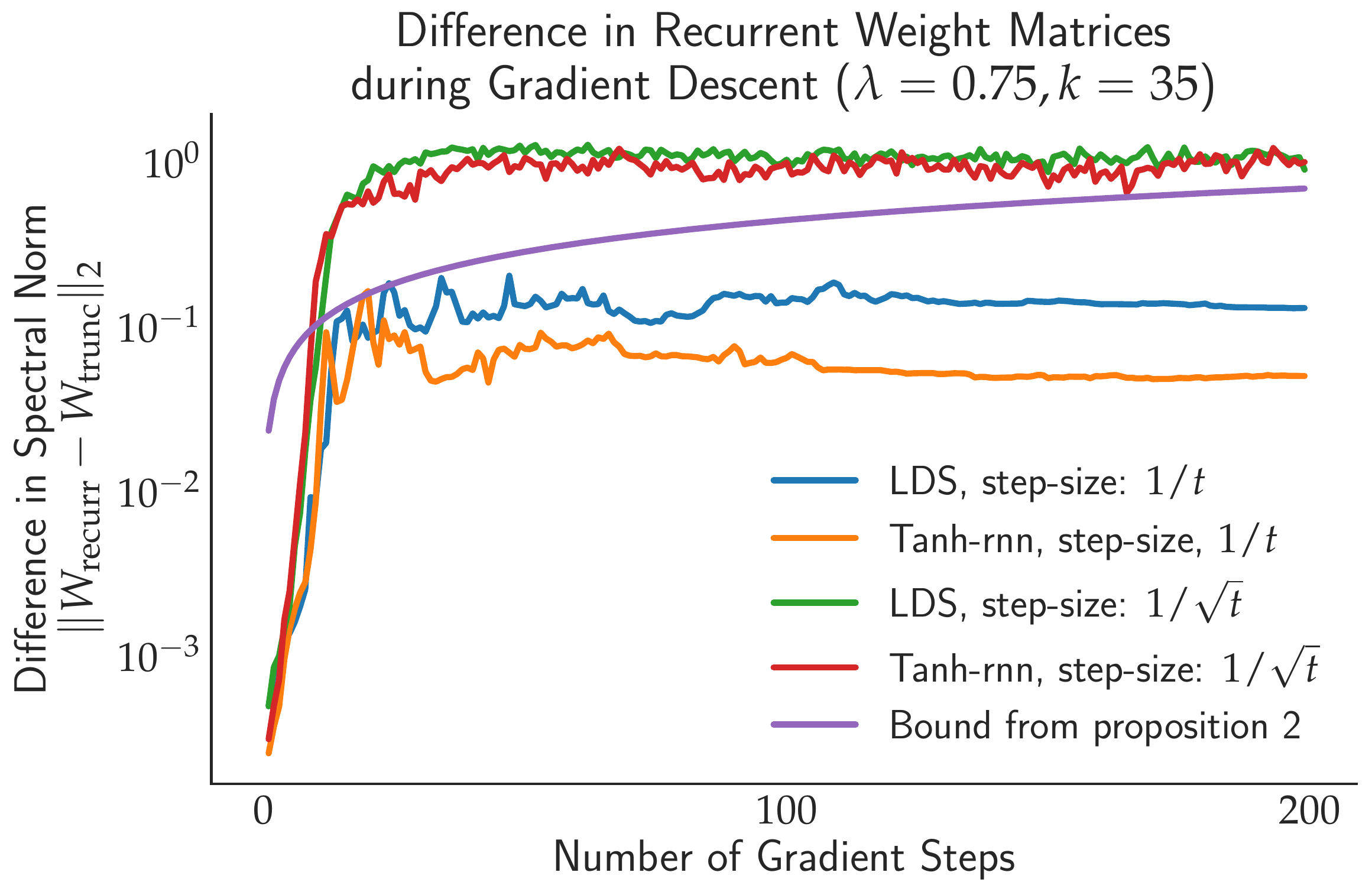}
    \caption{
        Empirical validation Proposition~\ref{prop:grad_descent_bound} on random
        Gaussian instances. Without the $1/t$ rate, the gradient descent bound
        no longer appears qualitatively correct, suggesting the $O(1/t)$ rate is
        necessary.}
    \label{fig:synthetic_grad_descent_bounds}%
\end{figure}

Concretely, we generate random problem instance by fixing a sequence length $T=200$,
sampling input data $x_t \simiid \normal{0}{4\cdot I_{32}}$, and sampling $y_T
\sim \mbox{Unif}[-2, 2]$.  Next, we set $\lambda = 0.75$ and randomly initialize
a stable linear dynamical system or RNN with $\tanh$ non-linearity by sampling
$U_{ij}, W_{ij} \simiid \normal{0}{0.5}$ and thresholding the singular values of
$W$ so $\norm{W} \leq \lambda$. We use the squared loss and prediction function
$f(h_t, x_t) = Ch_t + Dx_t$, where $C, D \simiid \normal{0}{I_{32}}$.  We fix
the truncation length to $k=35$, set the learning rate to $\alpha_t = \alpha/t$
for $\alpha = 0.01$, and take $N=200$ gradient steps.  
These parameters are chosen so that the $\gamma k\lambda^k N^{\alpha \beta + 1}$
bound from Proposition~\ref{prop:grad_descent_bound} does not become vacuous --
by triangle inequality, we always have $\norm{\tw - \rw} \leq 2\lambda$. 

\paragraph{Stable vs. unstable models.}
The word and character level language modeling experiments are based on
publically available code from \cite{merityRegOpt}. The polyphonic music modeling code is
based on the code in \cite{bai2018empirical},
and the slot-filling model is a reimplementation of \cite{mesnil2015using}
\footnote{The word-level
language modeling code is based on \url{https://github.com/pytorch/examples/tree/master/word\_language\_model},
the character-level code is based on
\url{https://github.com/salesforce/awd-lstm-lm}, and the polyphonic music
modeling code is based on \url{https://github.com/locuslab/TCN}.}

Since the sufficient conditions for stability derived in
Section~\ref{sec:examples} only apply for networks with a single layer, we use a
single layer RNN or LSTM for all experiments. Further, our theoretical results
are only applicable for vanilla SGD, and not adaptive gradient methods, so all
models are trained with SGD. Table~\ref{table:hyperparameters} contains a
summary of all the hyperparameters for each experiment.

\begin{table}
\caption{Hyperparameters for all experiments}
\label{table:hyperparameters}
\begin{center}
\begin{tabular}{llll}
    &&\multicolumn{2}{c}{\bf Model} \\
    &&\multicolumn{1}{c}{RNN} &\multicolumn{1}{c}{LSTM} \\
    \hline \\
    {\bf Word LM} & Number layers & 1 & 1 \\
                  & Hidden units  & 256 & 1024 \\
                  & Embedding size & 1024 & 512 \\
                  & Dropout        & 0.25 & 0.65 \\
                  & Batch size & 20 & 20 \\
                  & Learning rate & 2.0 & 20. \\
                  & BPTT          & 35 & 35 \\
                  & Gradient clipping & 0.25 & 1.0 \\
                  & Epochs & 40 & 40 \\
    {\bf Char LM} & Number layers & 1 & 1 \\
                  & Hidden units  & 768 & 1024 \\
                  & Embedding size & 400 &  400\\
                  & Dropout        & 0.1 & 0.1 \\
                  & Weight decay        & 1e-6 & 1e-6 \\
                  & Batch size & 80 & 80 \\
                  & Learning rate & 2.0 & 20.0 \\
                  & BPTT          & 150 & 150 \\
                  & Gradient clipping & 1.0 & 1.0 \\
                  & Epochs & 300 & 300 \\
    {\bf Polyphonic Music} & Number layers & 1 & 1 \\
                  & Hidden units  & 1024 & 1024 \\
                  & Dropout        & 0.1 & 0.1 \\
                  & Batch size & 1 & 1 \\
                  & Learning rate & 0.05 & 2.0 \\
                  & Gradient clipping & 5.0 & 5.0 \\
                  & Epochs & 100 & 100 \\
    {\bf Slot-Filling} & Number layers & 1 & 1 \\
                  & Hidden units  & 128 & 128 \\
                  & Embedding size & 64 &  64\\
                  & Dropout        & 0.5 & 0.5 \\
                  & Weight decay        & 1e-4 & 1e-4 \\
                  & Batch size & 128 & 128 \\
                  & Learning rate & 10.0 & 10.0 \\
                  & Gradient clipping & 1.0 & 1.0 \\
                  & Epochs & 100 & 100 \\
\end{tabular}
\end{center}
\end{table}

All hyperparameters are shared between the stable and unstable variants of both
models. In the RNN case, enforcing stability is conceptually simple, though
computationally expensive. Since $\tanh$ is 1-Lipschitz, the RNN is stable as
long as $\norm{W} < 1$. Therefore, after each gradient update, we project $W$
onto the spectral norm ball by taking the SVD and thresholding the singular
values to lie in $[0, 1)$. In the LSTM case, enforcing stability is conceptually
more difficult, but computationally simple. To ensure the LSTM is stable, we
appeal to Proposition~\ref{prop:lstm-stability}. We enforce the following inequalities after
each gradient update
\begin{enumerate}
\item The hidden-to-hidden forget gate matrix should satisfy $\norm{W_f}_\infty
< 0.128$, which is enforced by normalizing the $\ell_1$-
norm of each row to have value at most $0.128$. 
\item The input vectors $x_t$ must satisfy $\norm{x_t}_\infty \leq B_x = 0.75$,
which is achieved by thresholding all values to lie in $[-0.75, 0.75]$.
\item The bias of the forget gate $b_f$, must satsify $\norm{b_f}_\infty \leq
0.25$, which is again achieved by thresholding all values to lie in $[-0.25,
0.25]$.
\item The input-hidden forget gate matrix $U_f$ should satisfy
$\norm{U_f}_\infty \leq 0.25$. This is enforced by normalizing the $\ell_1$-
norm of each row to have value at most $0.25$. 
\item Given 1-4, the forget gate can take value at most $f_\infty < 0.64$.
Consequently, we enforce $\norm{W_i}_\infty, \norm{W_o}_\infty \leq 0.36$,
$\norm{W_z} \leq 0.091$, and $\norm{W_f}_\infty < \min\set{0.128, (1-0.64)^2}  =
0.128$.
\end{enumerate}
After 1-5 are enforced, by Proposition~\ref{prop:lstm-stability}, the resulting
(iterated)-LSTM is stable. Although the above description is somewhat
complicated, the implementation boils down to normalizing the rows of the LSTM
weight matrices, which can be done very efficiently in a few lines of PyTorch.

\paragraph{Data-dependent stability.}
Unlike the RNN, in an LSTM, it is not clear how to analytically compute the
stability parameter $\lambda$. Instead, we rely on a heuristic method to
estimate $\lambda$. Recall a model is stable if for all $x, h, h'$, we have
\begin{align}
    S(h, h', x) 
    := \frac{\norm{\phi_w(h, x) - \phi_w(h', x)}}{\norm{h - h'}}
    \leq \lambda < 1.
\end{align}
To estimate $\sup_{h, h', x} S(h, h', x)$, we do the following. First, we take
$x$ to be point in the training set. In the language modeling case, $x$ is one
of the learned word-vectors. We randomly sample and fix $x$, and then we perform
gradient ascent on $S(h, h', x)$ to find worst-case $h, h'$. In our experiments,
we initialize $h, h' \sim \normal{0}{0.1 \cdot I}$ and run gradient ascent with
learning rate $0.9$ for 1000 steps. This procedure is repeated 20 times, and we
estimate $\lambda$ as the maximum value of $S(h, h', x)$ encounted during any
iteration from any of the 20 random starting points.


\end{document}